\documentclass[11pt]{article}
\usepackage[utf8]{inputenc}
\usepackage[english]{babel}
\usepackage{epsfig,amsmath,amsthm,amssymb,latexsym,color}
\usepackage{graphicx,subfig,commath,fixmath}
\usepackage[shortlabels]{enumitem}
\usepackage{ifthen}
\usepackage{setspace}

\DeclareMathOperator*{\argmin}{arg\,min}

\newtheorem{theorem}{Theorem}[section]
\newtheorem{lemma}{Lemma}[section]

\newtheorem{algorithm}{Algorithm}
\newtheorem{assumption}{Assumption}

\newtheorem{definition}{Definition}[section]

\theoremstyle{remark}
\newtheorem{remark}{Remark}[section]
\usepackage{caption}
\usepackage{algorithm}
\usepackage{algpseudocode}

\usepackage{hyperref}

\begin{document}
\title{A Compressed Sensing Based Least Squares Approach to Semi-supervised Local Cluster Extraction} 
\author{Ming-Jun Lai 
\footnote{Department of Mathematics,
University of Georgia, Athens, GA 30602. mjlai@uga.edu. } 
\and 
Zhaiming Shen\footnote{Department of Mathematics,
University of Georgia, Athens, GA 30602. zhaiming.shen@uga.edu.}}
\maketitle
\begin{abstract} 
\noindent
A least squares semi-supervised local clustering algorithm based on the idea of compressed sensing is proposed to extract clusters from a 
graph with known adjacency matrix. The algorithm is based on a two-stage approach similar to the one in \cite{LaiMckenzie2020}.
However, under a weaker assumption and with less computational complexity 
than the one in \cite{LaiMckenzie2020}, 
the algorithm is shown to be able to find a desired cluster with high probability. The ``one cluster at a time" feature of our method distinguishes it from other global clustering methods.
Several numerical experiments are conducted on the synthetic data such as stochastic block model and real data such as MNIST, political blogs network, AT\&T and YaleB human faces data sets to demonstrate the effectiveness and efficiency of our algorithm. 
\end{abstract} 

\section{Introduction}
\label{intro}
Informally speaking, graph clustering aims at dividing the set of vertices from a graph into subsets in a way such that there are more edges within each subset, and fewer edges between different subsets.  When analyzing a graph, one of people's primary interest is to find the underlying clustered structure of the graph, as the vertices in the same cluster can reasonably be assumed to have 
some latent similarity. For data sets which are not presented as graphs, we can create a suitable auxiliary graph such as the $K$-nearest-neighbors ($K$-NN) graph based on the given data, for example, see \cite{Shi2000}, \cite{Mahoney2012} and \cite{Jacobs2018}. Then we can apply graph clustering techiques on this auxiliary graph. 

Graph clustering problem has become prevalent recently in areas of social network study \cite{Fortunato2010}, \cite{Hric2014} and \cite{Kossinets2016}, image classification  \cite{Camps2007}, \cite{Chen2005} and \cite{Shi2000}, natural language processing \cite{Dhillon2001} and \cite{Mihalcea2011}. For example, suppose a social network graph has vertices which represent users of a social network (e.g. Facebook, Linkedln), then the edges could represent users which are connected to each other. The sets of nodes with high inter-connectivity, which we call them communities or clusters, could represent friendship groups or co-workers. By identifying those communities we can suggest new connections to users. Note that some networks are directed (e.g. Twitter, Citation Networks), which could make community detection more subtle. For the scope of this paper, we will only focus on weighted undirected graphs.

The classical graph based clustering problem is a global clustering problem which assigns every vertice a unique cluster, assuming there are no multi-class vertices. It is usually considered as an unsupervised learning problem which can be done by using method such as spectral clustering \cite{Luxburg2007}, \cite{Ng2002} and \cite{Zelnik2004} or ways of finding an optimal cut of the graph \cite{Dhillon2004}, \cite{Ding2001}, these approaches are generally computational expensive and hard to implement for large data sets. It can also be done semi-supervisely, such as \cite{Kulis2005}, \cite{Jacobs2018} and \cite{Yin2018}. However, sometimes it is only of people's interests in finding one certain cluster which contains the target vertices, given some prior knowledge of a small portion of labels for the entire true cluster, which is usually attainable for real data. This type of problem is called local clustering, or local cluster extraction, which loosely speaking, is defined to be the problem which takes a set of vertices $\Gamma$ with given labels, we call them seeds, as input, and returns a cluster $C^{\#}$ such that $\Gamma\subset C^{\#}$. 
The local clustering problem hasn't been studied exhaustively, and many aspects of the local clustering problem still remain 
open. Some recent related work are \cite{Ha2020}, \cite{Yan2019}, \cite{Yin2017}, \cite{Veldt2019} and \cite{LaiMckenzie2020}. 
Especially the work in \cite{LaiMckenzie2020} is one of the recent works with the same problem setting as in this paper. 
More precisely, we propose a new semi-supervised local clustering approach using the ideas of compressed sensing and method of least squares to make the clustering effective and efficient. Indeed, as we will see in the numerical experiments section, our approach outperforms the work in \cite{LaiMckenzie2020} in terms of both the accuracy and 
efficiency. 

The main contribution of this paper is that it proposes the local cluster extraction Algorithms \ref{alg3} and \ref{alg4}  which improve the
performance of the algorithms in \cite{LaiMckenzie2020} and also slightly improve state-of-the-art result in \cite{Abbe2015} for the political blog network \cite{AG05}. It also achieves better or comparable results on synthetic stochastic block model, human faces data, and MNIST data compared with several other modern local clustering algorithms or semi-supervised algorithms.

The subsequent sections in this paper are structured as follows. In Section \ref{sectionprelim}, we give brief introductions to spectral clustering and concept of graph Laplacian, we also make the assumptions for the graph model which we will use later for theoretical analysis. In Section \ref{sectionalgorithms}, we explain the main algorithms for solving the local cluster extraction problem in two-stage and show the correctness of our algorithms asymptotically. In Section \ref{sectioncomplexity}, we analyze the computational complexity of our algorithms. In Section \ref{sectionnumerical}, several synthetic and real data sets are used to evaluate the performance of our algorithms and we also compared their performances with the state-of-the-art results.

\section{Preliminaries and Models}
\label{sectionprelim}
\subsection{Notations and Definitions}
We use standard notation $G=(V,E)$ to denote the graph $G$ with the set of vertices $V$ and set of edges $E$. For the case $|V|=n$, we identify $V$ with the set of integers $[n]:=\{1,2,\cdots,n\}$. We use $A$ to denote the adjacency matrix (possibly non-negative weighted) of $G$, so in the undirected case, $A$ is a symmetric matrix. Let $D$ be the diagonal matrix $D=diag(d_1,d_2,\cdots,d_n)$, where each $d_i$ is the degree of vertex $i$. We have the following definition.
\begin{definition}
	The \emph{unnormalized graph Laplacian} is defined as $L=D-A$. There are also two other \emph{normalized graph Laplacians} which are
	symmetric graph Laplacian $L_{sym}: =I-D^{-1/2}AD^{-1/2}$, and the random walk graph Laplacian $L_{rw}: = I-D^{-1}A$.
\end{definition} 
%Note that $L$ and $L_{rw}$ are only difference by left multiplying a diagonal matrix, and for the properties to our interests,  they will behave almost exactly the same. So for the discussions further on, we will use $L$ to denote both unnormalized graph Laplacian and random walk graph Laplacian whenever no confusion will be caused. \\
The following result serves as the foundation of our approach for solving the graph clustering problem, we omit its proof here by directly referring to \cite{Chung1997} and \cite{Luxburg2007}.
\begin{lemma} \label{kernelthm}
	Let G be an undirected graph of size $n$ with non-negative weights. Then the multiplicity $k$ of the eigenvalue $0$ of $L$ $ (L_{rw})$ equals to the number of connected components  $C_1, C_2, \cdots, C_k$ in $G$, and the indicator vectors $\textbf{1}_{C_1}, \cdots, \textbf{1}_{C_k}\in\mathbb{R}^n$ on these components span the kernel of $L$ $(L_{rw})$.
\end{lemma}
Let us introduce some more notations which we will use later. Suppose for the moment we have information about structure of the underlying clusters for each vertex, then it is useful to write $G$ as a union of two edge-disjoint subgraphs $G=G^{in}\cup G^{out}$ where $G^{in}=(V,E^{in})$ consists of only intra-connection edges, and $G^{out}=(V,E^{out})$ consists of only inter-connection edges. We will use $d_i^{in}$ to denote the degree of vertex $i$ in the subgraph $G^{in}$, and $d_i^{out}$ to denote the degree of vertex $i$ in the subgraph $G^{out}$. We will also use $A^{in}$ and $L^{in}$ to denote the adjacency matrix and graph Laplacian associated with $G^{in}$, and $A^{out}$ and $L^{out}$ to denote the adjacency matrix and graph Laplacian associated with $G^{out}$. Note that these notations are just for convenience for the analysis in the next section, in reality we will have no assurance about which cluster each individual vertex belongs to, so we will have no access to $A^{in}$ and $L^{in}$. It is also worthwhile to point out that $A=A^{in}+A^{out}$ but $L\neq L^{in}+L^{out}$ in general. Furthermore, we will use $|L|$ or $|\mathbf{y}|$ to denote the matrix or vector where each its entry is replaced by the absolute value, and we will use $|V|$ to denote the size of $V$ whenever $V$ is a set. 
In the later sections, we will use $L$ and $L^{in}$ to indicate $L_{rw}$ and $L^{in}_{rw}$ respectively, and use $L_{C}$ and $L^{in}_C$ to denote the submatrices of $L$ and $L^{in}$ with column indices subset $C\subset V=[n]$ respectively.
For convenience, let us formulate the notations being used through the paper into Table \ref{Notation}. 

\begin{table}[h]
    \centering
    \caption{Table of Notations}
    \label{Notation}
    \begin{tabular}{lll}
        \hline
         Symbols & \quad\quad &  \\
         \hline
         $G$ & \quad\quad & A general graph \\
         $|G|$ & \quad\quad & Size of G \\
        $V$ & \quad\quad & Set of vertices of graph G \\
        & \quad\quad & (We identify $V=\{1,2,\cdots,n\}$ if $|G|=n$ through the paper) \\
        $|V|$ & \quad\quad & Size of $V$ \\
        $E$ & \quad\quad & Set of edges of graph G \\
        $E^{in}$ & \quad\quad & Subset of $E$ which consists only intra-connection edges \\
        $E^{out}$ & \quad\quad & Subset of $E$ which consists only inter-connection edges \\
        $G^{in}$ & \quad\quad & Subgraph of $G$ on $V$ with edge set $E^{in}$ \\
        $G^{out}$ & \quad\quad & Subgraph of $G$ on $V$ with edge set $E^{out}$ \\
        $A$ & \quad\quad & Adjacency matrix of graph $G$ \\
        $A^{in}$ & \quad\quad & Adjacency matrix of graph $G^{in}$ \\
        $A^{out}$ & \quad\quad & Adjacency matrix of graph $G^{out}$ \\
        $L$ & \quad\quad & Random walk graph Laplacian of $G$ \\
        $L^{in}$ & \quad\quad & Random walk graph Laplacian of $G^{in}$ \\
        $L^{out}$ & \quad\quad & Random walk graph Laplacian of $G^{out}$  \\
        $\mathbf{1}_{C}$ & \quad\quad & Indicator vector on subset $C\subset V$ \\
        $L_{C}$ & \quad\quad & submatrix of $L$ with column indices $C\subset V$ \\
        $L^{in}_{C}$ & \quad\quad & submatrix of $L^{in}$ with column indices $C\subset V$ \\
        $|L|$ & \quad\quad & Entrywised absolute value operation on matrix $L$ \\ 
        $|\mathbf{y}|$ & \quad\quad & Entrywised absolute value operation on vector $\mathbf{y}$ \\
        $\|M\|$ & \quad\quad & $\|\cdot\|_2$ norm of matrix $M$ \\
        $\|\mathbf{y}\|$ & \quad\quad & $\|\cdot\|_2$ norm of vector $\mathbf{y}$.
         \\
       \hline
    \end{tabular}
\end{table}

\subsection{Graph Model Assumptions}
We make the following assumption for our graph model in the asymptotic perspective.
\begin{assumption}
Suppose $G=(V,E)$ can be partitioned into $k=O(1)$ connected components such that $V=C_1\cup\cdots\cup C_k$, where each $C_i$ is the underlying vertex set for each connected component of $G$.
\begin{itemize}
    \item[(I)] The degree of each vertex is asymptotically the same for vertices belong to the same cluster $C_i$.
    \item[(II)] The degree $d^{out}_i$ is small relative to degree $d^{in}_i$ asymptotically for each vertex $i\in V$.
\end{itemize}  
\end{assumption}
The random graphs which satisfies assumption (I) is not uncommon, for example, 
the Erd\"os-R\'enyi (ER) model $G(n,p)$ with $p\sim\frac{\omega(n)\log(n)}{n}$ for any 
$\omega(n)\to\infty$, see \cite{ER1959} and \cite{Chung2006}. A natural generalization of the ER model 
is the stochastic block model (SBM) \cite{Holland1983}, which is a generative model for random graphs 
with certain edge densities within and between underlying clusters, such that the edges within 
clusters are more dense than the edges between clusters. %Using the standard notation of $SSBM(n,k,p,q)$, where $n$ is the number of vertices, $k$ is the number of clusters, $p$ is the intra-connection probability, and $q$ is the inter-connection probability $q$.
In the case of each cluster has the same size and the intra- and inter-connection probability are the 
same among all the vertices, we have the symmetric stochastic block model (SSBM). It is worthwhile to 
note that the information theoretical bound for exact cluster recovery in SBM are given in \cite{Abbe2018} and \cite{Abbe2015}. 
It was also shown in \cite{LaiMckenzie2020} that a general SBM under certain assumptions of the parameters can be clustered by using a compressed sensing approach. Our model   
requires a weaker assumption than the one in \cite{LaiMckenzie2020}, indeed, we remove the assumption imposed on the eigenvalues of $L$ in \cite{LaiMckenzie2020}.  Therefore, our model  will be applicable 
to a broader range of random graphs.   

\section{Main Algorithms}
\label{sectionalgorithms}
Our analysis is based on the following key observation. Suppose that graph $G$ has $k$ connected 
components $C_1, \cdots, C_k$, i.e., $L=L^{in}$. Suppose further that we temporarily have access to the information 
about the structure of $L^{in}$. Then we can write the graph Laplacian $L^{in}$ into a block diagonal 
form
\begin{equation}
L=L^{in} = \left( \begin{array}{cccc}
L^{in}_1 &  &  &    \\
& L^{in}_2 &  &    \\
&  & \ddots &  \\
&  &  & L^{in}_k   \\
\end{array} \right).
\end{equation}
Suppose now we are interested in finding the cluster with the smallest number of vertices, say $C_1$, 
which corresponds to $L^{in}_1$. By Lemma \ref{kernelthm}, $\{\mathbf{1}_{C_1},\cdots, 
\mathbf{1}_{C_k}\}$ forms a basis of the kernel $W_0$ of $L$. Note that all the $\mathbf{1}_{C_i}$ 
have disjoint supports, so for $\mathbf{w}\in W_0$ and $\mathbf{w}\neq 0$, we can write 
\begin{equation}
\mathbf{w}=\sum_{i=1}^k \alpha_i \mathbf{1}_{C_i}
\end{equation}
with some $\alpha_i\neq 0$. 
Therefore, if $\mathbf{1}_{C_1}$ has the fewest non-zero entries among all elements of 
$W_0\setminus\{\mathbf{0}\}$, then we can find it by solving the following minimization problem:
\begin{equation} \label{constrainminnolabel}
\min||\mathbf{w}||_0 \quad \text{s.t.} \quad L^{in}\mathbf{w}=\mathbf{0} \quad \text{and} \quad \mathbf{w}
\neq \mathbf{0}. 
\end{equation}
Here the $\ell_0$ norm $\|\cdot\|_0$ indicates the number of nonzero components for the input vector. Problem (\ref{constrainminnolabel}) can be solved using method such as greedy algorithm in compressed sensing 
as explained in \cite{LaiMckenzie2020}.
However, we will propose a different approach to tackle it in this paper and 
demonstrate that the new approach is more effective numerically and require a fewer number of assumptions.

\subsection{Least Squares Cluster Pursuit}
Let us consider problem (\ref{constrainminnolabel}) again, instead of finding $C_1$ directly, let us try to find what are not 
in $C_1$. 
Suppose there is a superset $\Omega\subset V$ such that $C_1\subset\Omega$, and $C_i\not\subset\Omega$ for all $i=2,\cdots,n$. 
Since $L^{in}\mathbf{1}_{C_1}=\mathbf{0}$, we have
\begin{equation}
L^{in}\mathbf{1}_{\Omega}=L^{in}(\mathbf{1}_{\Omega\setminus C_1}+\mathbf{1}_{C_1})=L^{in}\mathbf{1}_{\Omega\setminus C_1}+L^{in}\mathbf{1}_{C_1}=L^{in}\mathbf{1}_{\Omega\setminus C_1}.
\end{equation}
Letting $\mathbf{y}:=L^{in}\mathbf{1}_{\Omega}$, then to find what are not in $C_1$ within $\Omega$ is equivalent to solve the following problem (\ref{lsqoriginal})
\begin{equation}  \label{lsqoriginal}
\argmin_{\mathbf{x}\in\mathbb{R}^n}\|L^{in}\mathbf{x}-\mathbf{y}\|_2.
\end{equation}
Note that solving problem (\ref{lsqoriginal}) directly will give $\mathbf{x}^*=\mathbf{1}_{\Omega}\in\mathbb{R}^n$ and $\mathbf{x}^*=\mathbf{1}_{\Omega\setminus C_1}\in\mathbb{R}^n$ both as solutions. By setting the columns $L^{in}_{V\setminus\Omega}=0$, solving problem (\ref{lsqoriginal}) is equivalent to solving 
\begin{equation} \label{rewritelsq}
\argmin_{\mathbf{x}\in\mathbb{R}^{|\Omega|}}\|L^{in}_{\Omega}\mathbf{x}-\mathbf{y}\|_2.
\end{equation}
Directly solving problem (\ref{rewritelsq}) gives at least two solutions 
$\mathbf{x}^*=\mathbf{1}\in\mathbb{R}^{|\Omega|}$ and $\mathbf{x}^*=\mathbf{1}_{C_1^\mathsf{c}}\in\mathbb{R}^{|\Omega|}$, where $C_1^\mathsf{c}$ indicates the complement set of $C_1$. Between these two solutions, the latter is much more informative for us to extract $C_1$ from $\Omega$ than the former. We need to find a way to 
avoid the non-informative solution $\mathbf{x}^*=\mathbf{1}$ but keep the informative solution
$\mathbf{x}^*=\mathbf{1}_{C_1^\mathsf{c}}$. 

We can achieve this by removing a subset of columns from index set $\Omega$.
Let us use $T\subset\Omega$ to indicate the indices of column we aim to remove. Suppose we could choose $T$ such that $T\subset C_1$. Now consider the following variation (\ref{lsqremoveT}) of the minimization problem (\ref{rewritelsq})
\begin{equation} \label{lsqremoveT}
\argmin_{\mathbf{x}\in \mathbb{R}^{|\Omega|-|T|}} \|L^{in}_{\Omega\setminus T}\mathbf{x}- \mathbf{y}\|_2.
\end{equation}
Different from solving (\ref{rewritelsq}) which gives two solutions, solving (\ref{lsqremoveT}) only gives one solution $\mathbf{x}^*=\mathbf{1}_{C_1^\mathsf{c}}$, as $\mathbf{x}^*=\mathbf{1}$ is no longer a solution because of the removal of $T$. The solution $\mathbf{x}^*=\mathbf{1}_{C_1^\mathsf{c}}$ is indeed still a solution to (\ref{lsqremoveT}) because $L^{in}_{\Omega\setminus T}\mathbf{1}_{C_1^\mathsf{c}}=L^{in}\mathbf{1}_{\Omega\setminus C_1}=0$.
Furthermore, the solution to (\ref{lsqremoveT}) is unique 
since it is a least squares problem with matrix $L^{in}_{\Omega\setminus T}$ of full column rank, therefore $\mathbf{x}^*=\mathbf{1}_{C_1^\mathsf{c}}$ is the unique solution to (\ref{lsqremoveT}). 

However, there is no way in theory we can select $T$ and assure the condition $T\subset C_1$. In practice, the way we choose $T$ is based on the following observation. Suppose $L=L^{in}$, $\Omega\supset C_1$ and $\Omega\not\supset C_i$ for $i=2,\cdots k$. Then $|L_{a}^{\top}|\cdot|\mathbf{y}|=0$ for all 
$a\in C_1$, and $|L_{a}^{\top}|\cdot|\mathbf{y}|>0$ for all $a\in\Omega\setminus C_1$. Therefore, we can choose $T$ in such a way that $|L_t^{\top}|\cdot|\mathbf{y}|$ is small for all $t\in T\subset\Omega$. These ideas lead to Algorithm \ref{alg:LSQ}.

\begin{algorithm}[h]
\caption{\textbf{Least Squares Cluster Pursuit} \label{alg:LSQ}}
\begin{algorithmic}
\Require 
Adjacency matrix $A$, vertex subset $\Omega\subset V$, least squares threshold parameter $\gamma\in (0,1)$, and rejection parameter $0.1\leq R\leq 0.9$.
\begin{enumerate}
\item Compute $L=I-D^{-1}A$ and $\mathbf{y}=L\mathbf{1}_{\Omega}$. 
\item Let $T$ be the set of column indices of $\gamma\cdot|\Omega|$ smallest components of the vector $|{L_{\Omega}}^{\top}|\cdot\mathbf{|y|}$.
\item Let $\mathbf{x}^\#$ be the solution to
\begin{equation} 
\label{inalglsqremoveT}
\argmin_{\mathbf{x}  \in \mathbb{R}^{|\Omega|-|T|}} \|L_{\Omega\setminus T}\mathbf{x}- \mathbf{y}\|_2
\end{equation}
obtained by using an iterative least squares solver.
\item Let $W^{\#} = \{i: \mathbf{x}_i^{\#}>R\}$  .  
\end{enumerate}
%\textbf{Output:} $C^{\#}_1=\Omega\setminus W^{\#}$. \\
\Ensure $C^{\#}_1=\Omega\setminus W^{\#}$. 
\end{algorithmic}
\end{algorithm}

\begin{remark}
We impose the absolute value rather than direct 
dot product in order to have fewer cancellation between vector components when summing over the entrywised products. In practice, the 
value of $\gamma\in (0,1)$ will not affect the performance too much as long as its value is not too extreme. We find that $0.15\leq\gamma\leq 0.4$ works well for our numerical experiments.
\end{remark}

\begin{remark}
In practice, we choose to use MATLAB's \emph{lsqr} function to solve the least squares problem (\ref{inalglsqremoveT}). As we will see in Lemma \ref{Cond}, our problem is well conditioned, so it is also possible to solve the normal equation exactly for problems which are not in a very large scale. However, we choose to solve it iteratively over exactly because the quality of the numerical solution is not essential for our task here, we are only interested in an approximated solution as we can use the cutoff $R$ number for clustering.
\end{remark}

\begin{remark}
As indicated in \cite{LaiMckenzie2020}, we can reformulate problem (\ref{constrainminnolabel}) as solving 
\begin{equation} \label{SparseLM}
\argmin_{\mathbf{x}  \in \mathbb{R}^n} \{\|L\mathbf{x}- \mathbf{y}\|_2: \quad \|x\|_0\leq s\}
\end{equation}   
by applying the greedy algorithms such as subspace pursuit \cite{Dai2009} and compressed sensing matching pursuit (CoSaMP) \cite{Needell2009}. 
Or alternatively, we can consider LASSO, see \cite{Santosa1986} and \cite{Tibshirani1996}, formulation of the problem 
\begin{equation} \label{SparseLasso}
\argmin_{\mathbf{x}  \in \mathbb{R}^n} \{\|L\mathbf{x}- \mathbf{y}\|^2_2+\lambda \|\mathbf{x}\|_1\}=\argmin_{\mathbf{x}  \in \mathbb{R}^n} \{\|L\mathbf{x}- \mathbf{y}\|^2_2+\lambda \|\mathbf{x}\|_0\}.
\end{equation}
The reason that Lasso is a good way to interpret this problem is that the solution $\mathbf{x}^{*}$ we are trying to solve for is the sparse indicator vector which satisfies $\|\mathbf{x}^*\|_1=\|\mathbf{x}^*\|_0$. We do not analyze it further here.
\end{remark}
However, in reality we have no access to $L^{in}$, what we know only is $L$, and in general $L\neq L^{in}$. We argue that the 
solution to the perturbed problem (\ref{inalglsqremoveT}) associated with $L$ will not be too far away from the solution to the unperturbed (\ref{lsqremoveT}) problem associated with $L^{in}$, if the difference between $L$ and 
$L^{in}$ is relative small. Let us make this precise by first quoting 
the following standard result in numerical analysis. 
%We omit the proof here by directly referring to  
%for the interested readers.

\begin{lemma} \label{NumAna}
Let $\|\cdot\|$ be an operator norm, $A\in\mathbb{R}^{n\times n}$ be a non-singular square matrix, $\mathbf{x}\in \mathbb{R}^n$, $\mathbf{y}\in\mathbb{R}^n$. Let $\tilde{A}$, $\mathbf{\tilde{x}}$, $\mathbf{\tilde{y}}$ be  perturbed measurements of $A$, $\mathbf{x}$, $\mathbf{y}$ respectively. Suppose $A\mathbf{x}=\mathbf{y}$, $\tilde{A}\mathbf{\tilde{x}}=\mathbf{\tilde{y}}$, and suppose further $cond(A)<\frac{\|A\|}{\|\tilde{A}-A\|}$, then 
\[\frac{\|\mathbf{\tilde{x}}-\mathbf{x}\|}{\|\mathbf{x}\|}\leq \frac{cond(A)}{1-cond(A)\frac{\|\tilde{A}-A\|}{\|A\|}}\Big(\frac{\|\tilde{A}-A\|}{\|A\|}+\frac{\|\mathbf{\tilde{y}}-\mathbf{y}\|}{\|\mathbf{y}\|}\Big).\]
\end{lemma}
\noindent
The above lemma asserts that the size of cond($A$) is significant in determining the stability of the solution $\mathbf{x}$ with respect to small perturbations on $A$ and $\mathbf{y}$. For the discussion from now on, we will use $\|\cdot\|$ to denote the standard vector or matrix induced two-norm $\|\cdot\|_2$ unless state otherwise. The next lemma claims the invertibility of $(L^{in}_{\Omega\setminus T})^{\top} L^{in}_{\Omega\setminus T}$ and gives an estimation bound of its condition number.

\begin{lemma} \label{Cond}
Let $V=\cup_{i=1}^k C_i$ be the disjoint union of $k=O(1)$ underlying clusters with size $n_i$ and assume $(I)$.
Let $d_j$ be the degree for vertex $j\in V=[n]$, $n_1=\min_{i\in [k]}n_i$, and suppose $\Omega\subset
V$ be such that $\Omega\supset C_1$ and $\Omega\not\supset C_i$ for $i=2,\cdots k$. Then
\begin{itemize}
\item[(i)] If $T\subset C_1$, then 
$(L^{in}_{\Omega\setminus T})^{\top}L^{in}_{\Omega\setminus T}$ is invertible.
\item[(ii)] Suppose further $\lceil\frac{3n_1}{4}\rceil\leq|T|< n_1$ and $\lceil\frac{5n_1}{4}\rceil\leq|\Omega|\leq\lceil\frac{7n_1}{4}\rceil
$.  Then 
\begin{equation*}
    cond\big((L^{in}_{\Omega\setminus T})^{\top} L^{in}_{\Omega\setminus T}\big)\leq 4
\end{equation*}
almost surely as $n_1\to \infty$, e.g. when $n\to \infty$. % and $k$ is a fixed constant independent of $n$. 
\end{itemize}
 
\end{lemma}

\begin{proof}
Without loss of generality, let us assume that the column indices of $L^{in}_{\Omega\setminus T}$ are already permuted such that the indices number 
is in the same order relative to their underlying clusters.
The invertibility of $(L^{in}_{\Omega\setminus T})^{\top} L^{in}_{\Omega\setminus T}$ follows directly from the fact that $L^{in}_{\Omega\setminus T}$ is of full column rank. So let us show that $L^{in}_{\Omega\setminus T}$ is of full column rank. Because of the reordering, $L^{in}_{\Omega\setminus T}$ is in a block diagonal form 
\begin{equation*}
L_{\Omega\setminus T}^{in} = \left( \begin{array}{cccc}
L^{in}_{C_1\setminus T} &  &  &    \\
& L^{in}_{\Omega\cap C_2} &  &    \\
& & L^{in}_{\Omega\cap C_3}  &    \\
&  & & \ddots  \\
\end{array} \right).
\end{equation*}
It is then suffices to show each block is of full column rank. By Lemma \ref{kernelthm}, each of $L^{in}_{C_i}$ has $\lambda=0$ as an eigenvalue with multiplicity one, and the corresponding eigenspace is spanned by $\mathbf{1}_{C_i}$. Hence $rank(L^{in}_{C_i})=|C_i|-1$. Now suppose by contradiction that the columns of $L^{in}_{C_1\setminus T}$ are linearly dependent, so there exists $\mathbf{v}\neq\mathbf{0}$ such that $L^{in}_{C_1\setminus T}\mathbf{v}=\mathbf{0}$, or $L^{in}_{C_1\setminus T}\mathbf{v} + L^{in}_T\cdot\mathbf{0}=\mathbf{0}$. This means that $\mathbf{u}=(\mathbf{v},\mathbf{0})$ is an eigenvector associated to eigenvalue zero, which contradicts the fact that the eigenspace is spanned by $\mathbf{1}_{C_i}$. Therefore $L^{in}_{C_1\setminus T}$ is linearly independent, hence $L^{in}_{C_1\setminus T}$ is of full column rank. For $C_i$ with $i\geq 2$, since $C_i\notin \Omega$, $\Omega\cap C_i$ is a proper subset of $C_i$. The strategy above applies as well. Therefore all blocks in $L^{in}_{\Omega\setminus T}$ are of full column rank, so $L^{in}_{\Omega\setminus T}$ is of full column rank.

Now since $(L^{in}_{\Omega\setminus T})^{\top} L^{in}_{\Omega\setminus T}$ 
is in a block form, to estimate the condition number, we only need to estimate the largest and smallest eigenvalues for 
each block. 
Writing $L^{in}_{\Omega\setminus T}=[l_{ij}]$ and $(L^{in}_{\Omega\setminus T})^{\top} L^{in}_{\Omega\setminus T}=[s_{ij}]$,  for 
each $i\in C_1\setminus T$, 
$s_{ii}=\sum_{k=1}^{n}l_{ki}l_{ki}=\sum_{k=1}^{n}l_{ki}^2=\sum_{k=1}^{n_1}l_{ki}^2=1+\frac{1}{d_i^{in}}$, 
and for $i, j\in C_1\setminus T$ with $i\not=j$, 
$s_{ij}=\sum_{k=1}^{n}l_{ki}l_{kj}=\sum_{k=1}^{n_1}l_{ki}l_{kj}$. 
Note that the probability of having an edge between $i$ and $j$ given degree sequences $d_1,\cdots d_{n_1}$ equals to $\frac{d_i d_j}{\sum_{i
\in C_1}d_i}$, as the existence of an edge between two vertices is proportional to their degrees. So $l_{ij}$ equals to $-\frac{1}{d_i}$ with probability $\frac{d_i d_j}{\sum_{i\in C_1}d_i}$, which implies $\mathbb{E}(l_{ij})=-\frac{d_j}{\sum_{i\in C_1}d_i}$; $l_{ji}$ equals to $-\frac{1}{d_j}$ with probability $\frac{d_i d_j}{\sum_{i\in C_1}d_i}$,
which implies $\mathbb{E}(l_{ji})=-\frac{d_i}{\sum_{i\in C_1}d_i}$. 
Hence the expectation 
\begin{align*}
\mathbb{E}(s_{ij})&=
\mathbb{E}(\sum_{k=1}^{n}l_{ki}l_{kj})=\sum_{k=1}^{n}\mathbb{E}(l_{ki})\mathbb{E}(l_{kj})
=\sum_{k=1}^{n_1}\mathbb{E}(l_{ki})\mathbb{E}(l_{kj}) \\
&=\frac{d_i d_j}{\sum_{i\in C_1}d_i}\cdot (-\frac{1}{d_i})+\frac{d_i d_j}{\sum_{i\in C_1}d_i}\cdot (-\frac{1}{d_j})+\frac{d_k d_i}{\sum_{i\in C_1}d_i}\cdot\frac{d_k d_j}{\sum_{i\in C_1}d_i}\cdot (\frac{1}{d_k})^2\\
& = -\frac{d_i+d_j}{\sum_{i\in C_1}d_i}+\frac{d_i d_j}{(\sum_{i\in C_1}d_i)^2} =-\frac{2}{n_1}+\frac{1}{n_1^2}.
\end{align*}
By the law of large numbers, 
$s_{ij} \to -\frac{2}{n_1}+\frac{1}{n_1^2}$ almost surely as $n_1\to\infty$.
Therefore for $i\in C_1\setminus T$, we have 
\begin{align*}
\sum_{j\in C_1\setminus T, j\neq i}|s_{ij}| \to |C_1\setminus T|\cdot 
(\frac{2}{n_1}-\frac{1}{n_1^2})\leq \frac{n_1}{4}\cdot (\frac{2}{n_1}-\frac{1}{n_1^2})\leq \frac{1}{2}   
\end{align*}
almost surely as $n_1\to\infty$.
Similarly, for each $i\in C_k\cap(\Omega\setminus C_1)$, $k\geq 2$, we have $s_{ii}=1+\frac{1}{d^{in}_i}$, and $\sum_{j\in 
C_k\cap(\Omega\setminus C_1),j\neq i}|s_{ij}|\to  \frac{n_1}{4}\cdot (\frac{2}{n_k}-\frac{1}{n_k^2})\leq \frac{1}{2}$ almost 
surely as $n_1\to\infty$. 
\\
Now we apply Gershgorin's circle theorem  to bound the spectrum of $(L^{in}_{\Omega\setminus T})^{\top} 
L^{in}_{\Omega\setminus T}$. 
For all $i\in \Omega\setminus T$, the circles are centered at $1+\frac{1}{d_i}$, with radius less than or equal to $\frac{1}{2}$ 
almost surely, hence $\sigma_{\min}((L^{in}_{\Omega\setminus T})^{\top}L^{in}_{\Omega\setminus T})\geq \frac{1}{2}$ and $\sigma_{\max}((L^{in}_{\Omega\setminus T})^{\top}L^{in}_{\Omega\setminus T})\leq \frac{3}{2}+\frac{1}{d_i}\leq 2$. 
almost surely. Therefore we have 
\begin{align*}
\text{cond}\big((L^{in}_{\Omega\setminus T})^{\top} L^{in}_{\Omega\setminus T}\big)=\frac{\sigma_{\max}((L^{in}_{\Omega\setminus 
T})^{\top}L^{in}_{\Omega\setminus T})}{\sigma_{\min}((L^{in}_{\Omega\setminus T})^{\top}L^{in}_{\Omega\setminus T})}\leq 4
\end{align*}
almost surely, as desired.
\end{proof}

\begin{remark}
Note that there is a minor difficulty in estimating the expectation of inner product between two different columns of $L^{in}_{\Omega\setminus T}$. The computation assumes the independence of degree distribution of each individual vertex within each cluster, but this may not be true in general for arbitrary graph. However, the independence will occur if the asymptotic uniformity of the degree distribution within each cluster is assumed, that is why our model needs this assumption. 
\end{remark}
\noindent
Now the perturbed problem (\ref{inalglsqremoveT}) is equivalent to solving $(L_{\Omega\setminus T}^{\top} L_{\Omega\setminus T})\mathbf{x}^{\#}=L_{\Omega\setminus T}^{\top}\mathbf{\tilde{y}}=L_{\Omega\setminus T}^{\top}(L\mathbf{1}_{\Omega})$, 
while the unperturbed problem (\ref{lsqremoveT}) is  to solve $(L^{in}_{\Omega\setminus T})^{\top} L^{in}_{\Omega\setminus T}\mathbf{x}^{*}=(L^{in}_{\Omega\setminus T})^{\top}\mathbf{y}=(L^{in}_{\Omega\setminus T})^{\top}(L^{in}\mathbf{1}_{\Omega})$. 
Let $M:=L-L^{in}$, $M_{\Omega}:=L_{\Omega}-L^{in}_{\Omega}$, and $M_{\Omega\setminus T}:=L_{\Omega\setminus T}-L^{in}_{\Omega\setminus T}$. Let us give an estimate for $M$.   
\begin{lemma} 
\label{Boundperturb}
Let $L$ be the graph Laplacian of $G$ and $M:=L-L^{in}$. Let $\epsilon_i:=\frac{d^{out}_i}{d_i}$
for all $i$ and $\epsilon_{max}:=\max_{i\in [n]} \epsilon_i$.  
Then $\|M\|\leq 2\epsilon_{\max}$.
\end{lemma}
\begin{proof}
Let $\delta_{ij}$ denote the Kronecker delta symbol, observe that 
\[L_{ij}:=\delta_{ij}-\frac{1}{d_i}A_{ij}=\delta_{ij}-\frac{1}{d^{in}_i+d^{out}_i}(A^{in}_{ij}+A^{out}_{ij}).\]
Since $\epsilon_i:=\frac{d^{out}_i}{d_i}$, we have $\frac{1}{d_i}=\frac{1}{d^{in}_i+d^{out}_i}=\frac{1}{d^{in}_i}-\frac{\epsilon_i}{d^{in}_i}$. So we have
\begin{align*}
L_{ij}&=\delta_{ij}-\Big(\frac{1}{d^{in}_i}-\frac{\epsilon_i}{d^{in}_i}\Big)(A^{in}_{ij}+A^{out}_{ij}) \\
&=\Big(\delta_{ij}-\frac{1}{d^{in}_i}A^{in}_{ij}\Big)-\frac{1}{d^{in}_i}A^{out}_{ij}+\frac{\epsilon_i}{d^{in}_i}(A^{in}_{ij}+A^{out}_{ij}) \\
&=L^{in}_{ij}-\frac{1-\epsilon_i}{d^{in}_i}A^{out}_{ij}+\frac{\epsilon_i}{d^{in}_i}A^{in}_{ij}.
\end{align*}
Therefore $M_{ij}=-\frac{1-\epsilon_i}{d^{in}_i}A^{out}_{ij}+\frac{\epsilon_i}{d^{in}_i}A^{in}_{ij}$. To bound the spectral norm we apply Gershgorin's circle theorem, noting that $M_{ii}=0$ for all $i$, hence
\begin{align*}
\|M\|&=\max\{|\lambda_i|: \lambda_i \quad \text{eigenvalue of}\quad M \}\leq \max_{i}\sum_{j}|M_{ij}| \\
&= \max_{i}\sum_j\Big|-\frac{1-\epsilon_i}{d^{in}_i}A_{ij}^{out} + \frac{\epsilon_i}{d^{in}_i} A_{ij}^{in}\Big| \\
&\leq \max_{i}\sum_j\Big|-\frac{1-\epsilon_i}{d^{in}_i}\Big|A_{ij}^{out} + \Big| \frac{\epsilon_i}{d^{in}_i}\Big| A_{ij}^{in} \\
&\leq \max_{i} \Big\{\frac{1-\epsilon_i}{d^{in}_i}\sum_{j}A^{out}_{ij}+\frac{\epsilon_i}{d^{in}_i}\sum_{j}A^{in}_{ij}\Big\}  \\
&=\max_{i} \Big\{\frac{1-\epsilon_i}{d^{in}_i}d^{out}_i+\frac{\epsilon_i}{d^{in}_i}d^{in}_i\Big\} 
=2\max_{i}\epsilon_i=2\epsilon_{\max}.
\end{align*}
This completes the proof.  
\end{proof}

\noindent
Next we will have the following result.
\begin{lemma}
\label{newlemma}
$\|(L^{in}_{\Omega\setminus T})^{\top}L^{in}_{\Omega}\mathbf{1}_\Omega \|
\geq \frac{\sqrt{|\Omega\setminus  C_1|}}{2}$ almost surely.  
\end{lemma} 
\begin{proof}
Note that $\|(L^{in}_{\Omega\setminus T})^{\top}(L^{in}\mathbf{1}_{\Omega})\| =  \|(L^{in}_{\Omega\setminus T})^{\top}L^{in}_{\Omega}\mathbf{1}\|$.  
We want to give an estimate of $\|(L^{in}_{\Omega\setminus T})^{\top}L^{in}_{\Omega}\mathbf{1}\|$. 
Similar to the computation we did in Lemma \ref{Cond}, for each $i\in C_1\setminus T$, we have $s_{ii}=1+\frac{1}{d^{in}_i}$, 
$\sum_{j\in C_1}s_{ij}=0$, and $\sum_{j\in \Omega\setminus C_1}s_{ij}=0$. For each $i\in C_k\cap(\Omega\setminus C_1)$, $k\geq 2$, 
we have $s_{ii}=1+\frac{1}{d^{in}_i}$, $\sum_{j\in C_1}s_{ij}=0$, and $\sum_{j\in C_k\cap(\Omega\setminus C_1),j\neq i}s_{ij}\to  
\frac{n_1}{4}\cdot (-\frac{2}{n_k}+\frac{1}{n_k^2})\geq -\frac{1}{2}$ almost surely. Therefore, the row sum of 
$(L^{in}_{\Omega\setminus T})^{\top}L^{in}_{\Omega}$ for row $i\in C_1\setminus T$ equals to zero, and the row sum 
$(L^{in}_{\Omega\setminus T})^{\top}L^{in}_{\Omega}$ for row $i\in \Omega\setminus C_1$ larger than $\frac{1}{2}$ almost surely. Hence $\|(L^{in}_{\Omega\setminus T})^{\top}L^{in}_{\Omega}\mathbf{1}\|\geq \frac{\sqrt{|\Omega\setminus 
C_1|}}{2}$ almost surely.  
\end{proof}

%\begin{theorem} \label{IndAna}
%Under the same assumptions as Lemma \ref{Cond}, let $\mathbf{x}^\#$ be the solution to the perturbed problem (\ref{inalglsqremoveT}), and $\mathbf{x}^*=\mathbf{1}_{C_1^\mathsf{c}}\in\mathbb{R}^{|\Omega|-|I|}$ which is the solution to the unperturbed problem (\ref{lsqremoveT}). Let $\epsilon_i:=d_i^{out}/d_i$, and $\epsilon_{\max}=\max_{i\in [n]}\epsilon_i$. Suppose $\epsilon_{\max}<3\times10^{-4}$, then
%\[\frac{\|\mathbf{x}^\#-\mathbf{x}^*\|}{\|\mathbf{x}^*\|}\leq\frac{cond\big((L^{in}_{\Omega\setminus I})^{\top} L^{in}_{\Omega\setminus I}\big)\cdot\Big(\frac{\big\|(L_{\Omega\setminus I})^{\top} L_{\Omega\setminus I}-(L^{in}_{\Omega\setminus I})^{\top} L^{in}_{\Omega\setminus I}\big\|}{\big\|(L^{in}_{\Omega\setminus I})^{\top} L^{in}_{\Omega\setminus I}\big\|}+\frac{\|(L_{\Omega\setminus I})^{\top}L\mathbf{1}_{\Omega}-(L^{in}_{\Omega\setminus I})^{\top}L^{in}\mathbf{1}_{\Omega}\|}{\|(L^{in}_{\Omega\setminus I})^{\top}L^{in}\mathbf{1}_{\Omega}\|}\Big)}{1-cond\big((L^{in}_{\Omega\setminus I})^{\top} L^{in}_{\Omega\setminus I}\big)\cdot\frac{\big\|(L_{\Omega\setminus I})^{\top} L_{\Omega\setminus I}-(L^{in}_{\Omega\setminus I})^{\top} L^{in}_{\Omega\setminus I}\big\|}{\big\|(L^{in}_{\Omega\setminus I})^{\top} L^{in}_{\Omega\setminus I}\big\|}}\leq 0.1.\]
%almost surely for large $n_1$.
%\end{theorem}

Now let us use previous lemmas to establish that the difference between perturbed solution and unperturbed solution is small in the order of $\epsilon_{\max}$.
\begin{theorem} 
\label{IndAna}
Under the same assumptions as Lemma \ref{Cond}, let $\mathbf{x}^\#$ be the solution 
to the perturbed problem (\ref{inalglsqremoveT}), and 
$\mathbf{x}^*=\mathbf{1}_{C_1^\mathsf{c}}\in\mathbb{R}^{|\Omega|-|T|}$ which is the solution to the 
unperturbed problem (\ref{lsqremoveT}). Then
\[\frac{\|\mathbf{x}^\#-\mathbf{x}^*\|}{\|\mathbf{x}^*\|}= O(\epsilon_{\max})\]
almost surely for large $n_1$.
\end{theorem}

\begin{proof}
Let
$B=(L^{in}_{\Omega\setminus T})^{\top} L^{in}_{\Omega\setminus T}$, $\tilde{B}=(L_{\Omega\setminus T})^{\top} L_{\Omega\setminus T}$, $\mathbf{y}=L^{in}\mathbf{1}_{\Omega}$, $\tilde{\mathbf{y}}=L\mathbf{1}_{\Omega}$. We will apply Lemma \ref{Cond} with $B$, $\tilde{B}$, $\mathbf{y}$, $\tilde{\mathbf{y}}$.

First by Lemma~\ref{Boundperturb}, we have $\|M\|\leq 2\epsilon_{\max}$. Therefore
\begin{align*}
\|\tilde{B}-B\|&=\big\|(L_{\Omega\setminus T})^{\top} L_{\Omega\setminus T}-(L^{in}_{\Omega\setminus T})^{\top} L^{in}_{\Omega\setminus T}\big\| \\
&=
\big\|(L^{in}_{\Omega\setminus T})^{\top} M_{\Omega\setminus T}+M_{\Omega\setminus T}^{\top}L^{in}_{\Omega\setminus T}+M_{\Omega\setminus T}^{\top}M_{\Omega\setminus T}\big\| \\
&\leq \big\|(L^{in}_{\Omega\setminus T})^{\top} M_{\Omega\setminus T}\big\|+\big\|M_{\Omega\setminus T}^{\top}L^{in}_{\Omega\setminus T}\big\|+\big\|M_{\Omega\setminus T}^{\top}M_{\Omega\setminus T}\big\|  \\
&\leq \big(2\|L^{in}_{\Omega\setminus T}\|+\|M_{\Omega\setminus T}\|\big)\cdot\|M_{\Omega\setminus T}\| \\
&\leq \big(2\|L^{in}_{\Omega\setminus T}\|+\|M\|\big)\cdot \|M\| \\
&\leq 4\epsilon_{\max}\cdot\big(\|L^{in}_{\Omega\setminus T}\|+\epsilon_{\max} \big).
\end{align*} 
For each $i\in \Omega\setminus T$, we have $\|L_i\|\geq 1$, and $\sigma_{\max}((L^{in}_{\Omega\setminus T})^{\top} L^{in}_{\Omega\setminus T})=\|(L^{in}_{\Omega\setminus T})^{\top} L^{in}_{\Omega\setminus T}\|=\sigma_{\max}^2(L^{in}_{\Omega\setminus T})=\|L^{in}_{\Omega\setminus T}\|^2\geq \max_{i\in\Omega\setminus T}\|L_i\|^2 \geq 1$. Hence
\begin{align}
\frac{\big\|(L_{\Omega\setminus T})^{\top} L_{\Omega\setminus T}-(L^{in}_{\Omega\setminus T})^{\top} L^{in}_{\Omega\setminus T}\big\|}{\big\|(L^{in}_{\Omega\setminus T})^{\top} L^{in}_{\Omega\setminus T}\big\|}
\leq &\frac{\big(2\|L^{in}_{\Omega\setminus T}\|+\|M\|\big)\cdot \|M\|}{\|L^{in}_{\Omega\setminus T}\|^2}\cr
\leq & \frac{4\epsilon_{\max}}{\|L^{in}_{\Omega\setminus T}\|} +\frac{4\epsilon_{\max}^2}{\|L^{in}_{\Omega\setminus T}\|^2} \cr
\leq & 4(\epsilon_{\max}+\epsilon_{\max}^2).
\end{align}
We also have
%\begin{align*}
%\|\mathbf{\tilde{y}}-\mathbf{y}\|&=\|(L_{\Omega\setminus I})^{\top}(L\mathbf{1}_{\Omega})-(L^{in}_{\Omega\setminus I})^{\top}(L^{in}\mathbf{1}_{\Omega})\|=\|(L_{\Omega\setminus I})^{\top}(L_{\Omega}\mathbf{1})-(L^{in}_{\Omega\setminus I})^{\top}(L^{in}_{\Omega}\mathbf{1})\| \\
%&\leq \|(L^{in}_{\Omega\setminus I})^{\top}M_{\Omega}+M_{\Omega\setminus I}^{\top}L^{in}_{\Omega}+M_{\Omega\setminus I}^{\top}M_{\Omega}\|\cdot|\Omega| \\
%&\leq |\Omega|\cdot\big(\|(L^{in}_{\Omega\setminus I})^{\top}M_{\Omega}\|+\|M_{\Omega\setminus I}^{\top}L^{in}_{\Omega}\|+\|M_{\Omega\setminus I}^{\top}M_{\Omega}\|\big) \\
%&\leq %|\Omega|\cdot\big(2\|L^{in}_{\Omega}\|+\|M_{\Omega}\|\big)\cdot \|M_{\Omega}\| \\
%&\leq |\Omega|\cdot\big(2\|L^{in}_{\Omega}\|+\|M\|\big)\cdot\|M\| \\
%&\leq |\Omega|\cdot\big(2\|L^{in}_{\Omega}\|+2\epsilon_{\max}\big)\cdot 2\epsilon_{\max}.
%\end{align*}
\begin{align*}
\|\tilde{\mathbf{y}}-\mathbf{y}\|&=\|(L_{\Omega\setminus T})^{\top}(L\mathbf{1}_{\Omega})-(L^{in}_{\Omega\setminus T})^{\top}(L^{in}\mathbf{1}_{\Omega})\| \\
&=\|(L^{in}_{\Omega\setminus T}+M_{\Omega\setminus T})^{\top}(L_{\Omega}\mathbf{1}_\Omega)-(L^{in}_{\Omega\setminus T})^{\top}(L^{in}_{\Omega}\mathbf{1}_\Omega)\| \\
&= \|\big((L^{in}_{\Omega\setminus T})^{\top}M_{\Omega}+M_{\Omega\setminus T}^{\top}L^{in}_{\Omega}+M_{\Omega\setminus T}^{\top}M_{\Omega}\big)\cdot\mathbf{1}_\Omega\| \\
&\leq \sqrt{|\Omega|}\cdot\big(\|(L^{in}_{\Omega\setminus T})^{\top}M_{\Omega}\|+\|M_{\Omega\setminus T}^{\top}L^{in}_{\Omega}\|+\|M_{\Omega\setminus T}^{\top}M_{\Omega}\|\big) \\
&\leq \sqrt{|\Omega|}\cdot\big(2\|L^{in}_{\Omega}\|+\|M_{\Omega}\|\big)\cdot \|M_{\Omega}\| \\
&\leq 4\sqrt{|\Omega|}\cdot\big(\|L^{in}_{\Omega}\|+\epsilon_{\max}\big)\cdot \epsilon_{\max}.
\end{align*}
Next by Lemma \ref{newlemma}, 
$\|(L^{in}_{\Omega\setminus T})^{\top}L^{in}_{\Omega}\mathbf{1}_\Omega \|\geq \frac{\sqrt{|\Omega\setminus 
C_1|}}{2}$ almost surely. Therefore
\begin{align*}
\frac{\|(L_{\Omega\setminus T})^{\top}L\mathbf{1}_{\Omega}-(L^{in}_{\Omega\setminus T})^{\top}L^{in}\mathbf{1}_{\Omega}\|}{\|(L^{in}_{\Omega\setminus T})^{\top}L^{in}\mathbf{1}_{\Omega}\|}&\leq \frac{4\sqrt{|\Omega|}\cdot\big(\|L^{in}_{\Omega}\|+\epsilon_{\max}\big)\cdot \epsilon_{\max}}{\sqrt{|\Omega\setminus C_1|}/2} \\
&\leq 8\sqrt{5}\epsilon_{\max}\cdot\big(\|L^{in}_{\Omega}\|+\epsilon_{\max}\big) \\
&\leq 8\sqrt{5}\epsilon_{\max}\cdot\big(\sqrt{2}+\epsilon_{\max}\big) \\
&=8\sqrt{10}\epsilon_{\max}+8\sqrt{5}\epsilon_{\max}^2.
\end{align*}
The second inequality holds since $|\Omega|\geq \lceil \frac{5n_1}{4}\rceil$.
The third inequality holds since $\sigma_{\max}((L^{in}_{\Omega})^{\top}L^{in}_{\Omega})\leq 2$, which comes from the similar reasoning as in Lemma \ref{Cond} by using Gershgorin's circle theorem. Consequently, we have $\|L^{in}_{\Omega}\|\leq \sqrt{2}$.
Now putting Lemma \ref{Cond} and Lemma \ref{NumAna} together with $B=(L^{in}_{\Omega\setminus T})^{\top} L^{in}_{\Omega\setminus T}$, $\tilde{B}=(L_{\Omega\setminus T})^{\top} L_{\Omega\setminus T}$, $\mathbf{y}=L^{in}\mathbf{1}_{\Omega}$, $\tilde{\mathbf{y}}=L\mathbf{1}_{\Omega}$, we have
\begin{align*}
\frac{\|\mathbf{x}^\#-\mathbf{x}^*\|}{\|\mathbf{x}^*\|}&\leq\frac{cond\big((L^{in}_{\Omega\setminus T})^{\top}L^{in}_{\Omega\setminus T}\big)\cdot\big(4\epsilon_{\max}+4\epsilon_{\max}^2+8\sqrt{10}\epsilon_{\max}+8\sqrt{5}\epsilon_{\max}^2\big)}{1-cond\big((L^{in}_{\Omega\setminus T})^{\top} L^{in}_{\Omega\setminus T}\big)\cdot\big(4\epsilon_{\max}+4\epsilon_{\max}^2\big)} \\
&\leq \frac{16\big((1+2\sqrt{10})\epsilon_{\max}+(1+2\sqrt{5})\epsilon_{\max}^2\big)}{1-16\epsilon_{\max}(1+\epsilon_{\max})}=O(\epsilon_{\max}).
\end{align*}
\end{proof}
%\begin{remark}
%The step $\|(L^{in}_{\Omega\setminus I})^{\top}L^{in}_{\Omega}\mathbf{1}\|\approx \|(L^{in}_{\Omega})^{\top}L^{in}_{\Omega}\mathbf{1}\|$ in the proof above is not very rigorous. However, in practice, if we choose $I$ to be a farily small subset of $\Omega$ based on the heuristic described previously, then $\|(L^{in}_{\Omega\setminus I})^{\top}L^{in}_{\Omega}\mathbf{1}\|$ and $\|(L^{in}_{\Omega})^{\top}L^{in}_{\Omega}\mathbf{1}\|$ will have the same magnitude.
%\end{remark}
\noindent
Next we can estimate the size of the symmetric difference between output $C_1^{\#}$ and the true cluster $C_1$ relative to the size of $C_1$, the symmetric difference is defined as $C_1^{\#}\triangle C_1: = (C_1^{\#}\setminus C_1) \cup (C_1\setminus C_1^{\#})$.
Let us state another lemma before we establish the result.
\begin{lemma} \label{indicatornorm}
Let $T\subset [n]$, $\mathbf{v}\in\mathbb{R}^{n}$, and $W^{\#}=\{i:\mathbf{v}_{i}>R \}$. Suppose 
$\|\mathbf{1}_{T}-\mathbf{v}\|\leq D$, then $|T\triangle W^{\#}|\leq \frac{D^2}{\min\{R^2, (1-R)^2\}}$.
\end{lemma}
\begin{proof}
Let $U^{\#}=[n]\setminus W^{\#}$ and write $\mathbf{v}=\mathbf{v}_{U^{\#}}+\mathbf{v}_{W^{\#}}$, 
where $\mathbf{v}_{U^{\#}}$ and $\mathbf{v}_{W^{\#}}$ are the parts of $\mathbf{v}$ supported on 
$U^{\#}$ and $W^{\#}$. Then we can write 
\[\|\mathbf{1}_{T}-\mathbf{v}\|^2=\|\mathbf{1}_{T\cap W^{\#}}-(\mathbf{v}_{W^{\#}})_{T\cap W^{\#}}
\|^2+\|(\mathbf{v}_{W^{\#}})_{W^{\#}\setminus T}\|^2+\|\mathbf{1}_{T\setminus W^{\#}} 
-\mathbf{v}_{U^{\#}}\|^2.\]
Note that $\|(\mathbf{v}_{W^{\#}})_{W^{\#}\setminus T}\|^2\geq R^2\cdot|W^{\#}\setminus T|$ and 
$\|\mathbf{1}_{T\setminus W^{\#}}-\mathbf{v}_{U^{\#}}\|^2\geq (1-R)^2\cdot|T\setminus W^{\#}|$. We 
have
\begin{align*}
%     R^2\cdot|T\triangle W^{\#}|&=R^2\cdot(|T\setminus W^{\#}|+|W^{\#}\setminus T|) \\
%     &\leq \|(\mathbf{v}_{W^{
% \#}})_{W^{\#}\setminus T}\|^2+\|\mathbf{1}_{T\setminus W^{\#}}-\mathbf{v}_{U^{\#}}\|^2 \\
% &\leq \|\mathbf
% {1}_{T}-\mathbf{v}\|^2.
\|\mathbf{1}_{T}-\mathbf{v}\|^2 &\geq \|(\mathbf{v}_{W^{
\#}})_{W^{\#}\setminus T}\|^2+\|\mathbf{1}_{T\setminus W^{\#}}-\mathbf{v}_{U^{\#}}\|^2  \\
& \geq R^2\cdot|W^{\#}\setminus T|+(1-R)^2\cdot|T\setminus W^{\#}| \\
& \geq \min\{R^2, (1-R)^2\}\cdot(|W^{\#}\setminus T|+|T\setminus W^{\#}|) \\
& = \min\{R^2, (1-R)^2\}\cdot|T\triangle W^\#|.
\end{align*}
Therefore $|T\triangle W^{\#}|\leq \frac{\|\mathbf{1}_{T}-\mathbf{v}\|^2}{\min\{R^2, (1-R)^2\}}\leq\frac{D^2}{\min\{R^2, (1-R)^2\}}$ as desired.
\end{proof}

%\begin{theorem} \label{Diffset}
%Under the same assumptions as Theorem \ref{IndAna}, we have 
%\[\frac{|C^{\#}_1\triangle C_1|}{|C_1|}=o(1).\]
%\end{theorem}
%\begin{proof}
%From Theorem \ref{IndAna}, we have $\|\mathbf{x}^{\#}-\mathbf{x}^{*}\|=\|\mathbf{x}^{\#}-\mathbf{1}_{\Omega\setminus C_1}\|\leq o(1)\cdot\|\mathbf{x}^{*}\|\leq o(\sqrt{n_1})$. Then by Lemma \ref{indicatornorm}, we get $|W^{\#}\triangle(\Omega\setminus C_1)|\leq o(n_1)$. Since $C_1^{\#}=\Omega\setminus W^{\#}$, it then follows $|C_1^{\#}\triangle C_1|\leq o(n_1)$, hence $\frac{|C^{\#}_1\triangle C_1|}{|C_1|}=o(1)$ as desired.
%\end{proof}

\begin{theorem} \label{Diffset}
Under the same assumptions as Theorem \ref{IndAna}, we have 
\[\frac{|C^{\#}_1\triangle C_1|}{|C_1|}\leq O(\epsilon^2_{\max}).\]
In other words, the error rate of successfully recovering $C_1$ is at most a constant multiple of $\epsilon^2_{\max}$.
\end{theorem}
\begin{proof}
	From Theorem \ref{IndAna}, we have $\|\mathbf{x}^{\#}-\mathbf{x}^{*}\|=\|\mathbf{x}^{\#}-\mathbf{1}_{\Omega\setminus C_1}\|\leq O(\epsilon_{\max})\cdot\|\mathbf{x}^{*}\|\leq O(\epsilon_{\max}\sqrt{n_1})$. By Lemma \ref{indicatornorm}, we get $|W^{\#}\triangle(\Omega\setminus C_1)|\leq O(\epsilon_{\max}^2 n_1)$. Since $C_1^{\#}=\Omega\setminus W^{\#}$, it then follows $|C_1^{\#}\triangle C_1|\leq O(\epsilon^2_{\max} n_1)$, hence $\frac{|C^{\#}_1\triangle C_1|}{|C_1|}=O(\epsilon_{\max}^2)$ as desired.
\end{proof}

%\begin{remark}
%If we take $R=0.5$ and $\epsilon_{\max}=10^{-3}$, then by using the estimate of the difference between $\mathbf{x}^{\#}$ and $\mathbf{x}^{*}$ specifically in the proof of Theorem \ref{IndAna}, we have $\frac{\|\mathbf{x}^\#-\mathbf{x}^*\|}{\|\mathbf{x}^*\|}=0.05$, and therefore $\frac{|C^{\#}_1\triangle C_1|}{|C_1|}\leq 0.01$, so the error rate is at most one percent in this case.
%\end{remark}

\subsection{Random Walk Threshold}
In order to apply Algorithm \ref{alg:LSQ}, we need a ``nice" superset which contains $C_1$. The task for this subsection is to find such a superset $\Omega$ from the given seeds $\Gamma$. We will apply a simple diffusion based random walk algorithm on $G$ to find such $\Omega$. This leads to Algorithm \ref{alg2}, which is described in \cite{LaiMckenzie2020} as well. However, the difference of the random walk threshold algorithm between this paper and the one in \cite{LaiMckenzie2020} is that the threshold parameter $\delta$ here is heuristically chosen to be larger than the corresponding threshold parameter in 
\cite{LaiMckenzie2020}. This is another advantage of our method as it will increase the chances of having $C_1$ entirely contained in $\Omega$. Such a choice is made based on the natural differences of our approaches. It is worthwhile to point out that there are also other sophisticated algorithms such as the ones described in \cite{Andersen2007}, \cite{Kloster2014} and \cite{Wang2017} which can achieve the same goal. We avoid using these methods here as our purpose is just to implement a fast way of obtaining a set $\Omega\supset C_1$. 

\begin{algorithm}[htbp]
\caption{\textbf{Random Walk Threshold}}
\label{alg2}
\begin{algorithmic}
\Require 
 Adjacency matrix $A$, a random walk threshold parameter $\delta\in (0,1)$, a set of seed vertices $\Gamma\subset C_1$, estimated size $\hat{n}_1\approx |C_1|$, and depth of random walk $t\in\mathbb{Z}^{+}$.
\begin{enumerate}
\item Compute $P=AD^{-1}$ and $\mathbf{v}^{(0)}=D\mathbf{1}_{\Gamma}$.
\item Compute $\mathbf{v}^{(t)}=P^t\mathbf{v}^{(0)}$.
\item Define $\Omega={\mathcal{L}}_{(1+\delta)\hat{n}_1}(\mathbf{v}^{(t)})$.  
\end{enumerate}
\Ensure $\Omega=\Omega\cup\Gamma$. 
\end{algorithmic}
\end{algorithm}
\noindent
The thresholding operator $\mathcal{L}_s(\mathord{\cdot})$ is defined as 
\[\mathcal{L}_s(\mathbf{v}):=\{i\in[n]: \text{$v_i$ among $s$ largest entries in $\mathbf{v}$}\}.\]
The motivation of Algorithm \ref{alg2} is the following intuitive observation. Suppose we are given seed vertices $\Gamma\subset C_1$, then by starting from $\Gamma$, since the edges within each cluster are more dense than those between different clusters, the probability of staying within $C_1$ will be much higher than entering other clusters $C_i$, for $i\neq 1$, in a short amount of depth. Therefore, by performing a random walk up to a certain depth $t$, e.g., $t=3$, we will have a well approximated set $\Omega$ such that $C_1$ is almost surely contained in $\Omega$. Let us make this more precisely in Theorem \ref{ThmRWT}.

\begin{theorem} \label{ThmRWT}
Assume $|\Gamma|=O(1)$ and $t=O(1)$ in Algorithm \ref{alg2}, the probability 
$\mathbb{P}\big(C_1\subset\Omega\big)\geq\mathbb{P}\big(\sum_{j\in 
C_1}\mathbf{v}^{(t)}_{j}=\|\mathbf{v}^{(t)}\|_1\big)\geq 1-O(\epsilon_{\max})$. In other words, the 
probability that the $t$-steps random walk with seed vertices $\Gamma$ being not in $C_1$ is at most a 
constant multiple of $\epsilon_{\max}$.
\end{theorem}
\begin{proof}
Let us first consider the case $|\Gamma|=1$. Suppose $\Gamma=\{s\}$. Then we have 
$\mathbb{P}\big(\sum_{j\in  C_1}\mathbf{v}^{(0)}_j=\|\mathbf{v}^{(0)}\|_1\big)=\mathbb{P}(\mathbf{v}^{(0)}_s
=\|\mathbf{v}^{(0)}\|_1)=1$. It is also easy to see that $\mathbb{P}\big(\sum_{j\in  
C_1}\mathbf{v}^{(1)}_j=\|\mathbf{v}^{(1)}\|_1\big)=d_i^{in}/d_i=1-\epsilon_i\geq 1-\epsilon_{\max}$.  
For $t\geq 2$, we have $\mathbb{P}\big(\sum_{j\in  
C_1}\mathbf{v}^{(t)}_j=\|\mathbf{v}^{(t)}\|_1\big)\geq(1-\epsilon_{\max})\cdot\mathbb{P}
\big(\sum_{j\in C_1}\mathbf{v}^{(t-1)}_j=\|\mathbf{v}^{(t-1)}\|_1\big)$. So by assuming 
$\mathbb{P}\big(\sum_{j\in C_1}\mathbf{v}^{(t-1)}_j=\|\mathbf{v}^{(t-1)}\|_1\big)\geq  
(1-\epsilon_{\max})^{t-1}\geq 1-(t-1)\epsilon_{\max}$, we have $\mathbb{P}\big(\sum_{j\in  
C_1}\mathbf{v}^{(t)}_j=\|\mathbf{v}^{(t)}\|_1\big)\geq (1-\epsilon_{\max})^{t}\geq  
1-t\epsilon_{\max}=1-O(\epsilon_{\max})$.

Suppose now $|\Gamma|>1$, we can apply the above argument to each individual vertex in $\Gamma$, where 
the random walk starting from each vertex can be considered independently, therefore we have 
$\mathbb{P}\big(\sum_{j\in C_1}\mathbf{v}^{(t)}_{j}=\|\mathbf{v}^{(t)}\|_1\big)\geq 
(1-t\epsilon_{\max})^{|\Gamma|}\geq 1-t\epsilon_{\max}|\Gamma|=1-O(\epsilon_{\max})$.
\end{proof}

\begin{remark}
	It is worthwhile to note that we do not want $t$ to be too large, one reason is that Theorem \ref{ThmRWT} tells us the probability of staying within the target cluster $C_1$ decreases as $t$ increases. An alternative interpretation is that we can treat our graph $G$, suppose connected, as a time homogeneous finite state Markov chain with evenly distributed transition probability determined by the vertex degree between adjacent vertices. Since $G$ is connected, it is certainly irreducible and aperiodic. By the fundamental theorem of Markov chains, the limiting probability of finally being at each vertex will be the same, regardless of what the seed set $\Gamma$ is. 
	
	Meanwhile, we do not want $t$ to be too small as well, otherwise the random walk will not be able to explore all the reachable vertices. 
	There is also a trade-off between the size of $\Gamma$ and the random walk depth $t$, where a smaller size of $\Gamma$ usually induces a larger $t$ in order to fully explore the target cluster.
\end{remark}

\subsection{Local Cluster Extraction}
Let us now combine the previous two subroutines into Algorithm \ref{alg3}. In practice, we may want to vary the number of iterations $MaxIter$ based on the number of examples in the data set in order to achieve a better performance. For the purpose theoretical analysis, let us fix $MaxIter=1$.

\begin{algorithm}[h]
\caption{\textbf{Least Squares Clustering (LSC)}}
\label{alg3}
\begin{algorithmic}
\Require 
Adjacency matrix $A$, a random walk threshold parameter $\delta\in (0,1)$, a set of seed vertices $\Gamma\subset C_1$, estimated size $\hat{n}_1\approx |C_1|$, depth of random walk $t\in\mathbb{Z}^{+}$, least squares parameter $\gamma\in (0,0.8)$, and rejection parameter $R\in[0,1)$. 
\begin{enumerate}
\item{}  \textbf{for} $i=1,\cdots, MaxIter$
\item{} \quad $\Omega \longleftarrow$ \textbf{Random Walk Threshold ($A$, $\Gamma$, $\hat{n}_1$, $\epsilon$, $t$)}.
\item{} \quad $\Gamma \longleftarrow$ \textbf{Least Squares Cluster Pursuit ($A$, $\Omega$, $R$, $\gamma$)}.
\item{} \textbf{end}
\item{} Let $C_1^{\#}=\Gamma$.
\end{enumerate}
\Ensure $C_1^{\#}$.
\end{algorithmic}
\end{algorithm}

\begin{remark}
The hyperparameter $MaxIter$ in the algorithm is usually choosen based on the size of initial seed vertices $\Gamma$ relative to $n$, we do not have a formal way of choosing the best $MaxIter$ rather than choose it heuristically. In practice, we believe $MaxIter\leq 3$ will do a very good job most of the time.
\end{remark}
The analysis in previous two subsections gives that the difference between true cluster $C_1$ and the 
estimated $C_1^{\#}$ is relative small compared to the size of $C_1$, this can be written more formally using the asymptotic notation.

\begin{theorem}
Suppose $\epsilon_{\max}=o(1)$ and $MaxIter=1$, then under the assumptions of Theorem \ref{Diffset} and \ref{ThmRWT}, we have $\mathbb{P}\Big(\frac{|C^{\#}_1\triangle C_1|}{|C_1|}\leq o(1)\Big)=1-o(1)$.
\end{theorem}
\begin{proof}
By Theorem \ref{ThmRWT}, we know that the probability of $\Omega\supset C_1$ after performing Algorithm \ref{alg2} is $1-O(\epsilon_{\max})=1-o(1)$. By Theorem \ref{Diffset}, the error rate is at 
most a constant multiple of $\epsilon^2_{\max}$ after performing Algorithm \ref{alg:LSQ}. Putting them together, we have 
$\mathbb{P}\Big(\frac{|C^{\#}_1\triangle C_1|}{|C_1|}\leq o(1)\Big)=1-o(1)$. 
\end{proof}

\subsection{From Local to Global}
We can make one step further by applying Algorithm \ref{alg3} iteratively on the entire graph to extract all the underlying clusters. 
That is, we remove $C_i^{\#}$ each time after the Algorithm \ref{alg3} finds it, and update the graph $G$ by removing the subgraph spanned 
by vertices $C_i^{\#}$ successively. This leads to Algorithm \ref{alg4}. We will not 
analyze further the theoretical guarantees of the iterative version the algorithm, but rather provide with numerical examples in 
the later section to show its effectiveness and efficiency.
\begin{algorithm}[ht]
\caption{\textbf{Iterative Least Squares Clustering (ILSC)}}
\label{alg4}
\begin{algorithmic}
\Require 
Adjacency matrix $A$, random walk threshold parameter $\delta\in (0,1)$, least squares parameter $\gamma\in (0,0.8)$, rejection 
parameter $R\in[0,1)$, depth of random walk $t\in\mathbb{Z}^{+}$. Seed vertices for each cluster $\Gamma_i\subset C_i$, estimated 
size $\hat{n}_i\approx |C_i|$ for $i=1,\cdots k$. 
\begin{enumerate}
\item{} \textbf{for} $i= 1,\cdots, k$
\item{} \quad Let $C_i^{\#}$ be the output of \textbf{Algorithm 3}.
\item{} \quad  Let $G^{(i)}$ be the subgraph spanned by $C_i^{\#}$.
\item{} \quad Updates $G\leftarrow G\setminus G^{(i)}$.
\item{} \textbf{end}
\end{enumerate}
\Ensure $C_1^{\#}, \cdots, C_k^{\#}$. 
\end{algorithmic}
\end{algorithm}

\begin{remark}
It is worth noting that Algorithm \ref{alg4} extracts one cluster at a time, which is different from most of other global unsupervised clustering algorithms. In practice, those global clustering methods could have impractically high run time \cite{Mossel2018} or tricky to implement \cite{Abbe2015}. In contrast, our method requires much lower computational time and can be implemented easily. In addition, the "one cluster at a time" feature of our method provides more flexibility for problems under certain circumstances.  
\end{remark}

\section{Computational Complexity}
\label{sectioncomplexity}
In this section, let us discuss the run time of the algorithms introduced previously.
\begin{theorem}
Algorithm \ref{alg2} requires $O(nd_{\max}t+n\log(n))$ operations, where $t$ is the depth of the random walk. 
\end{theorem}
\begin{proof}
Notice that if $A, D, P$ are stored as sparse matrices, then for each $t$ in the second step of Algorithm \ref{alg2},
it requires $O(nd_{\max})$, where $d_{\max}$ is the maximal degrees among all the vertices. 
Therefore the algorithm 
requires $O(nd_{\max}t+n\log(n))$, where the $O(n\log(n))$ part comes 
from the third step of sorting. In practice, the random walk depth $t$ is $O(1)$ with respect to the graph size $n$, therefore we have $O(nd_{\max}+n\log(n))$.
\end{proof}
\begin{theorem}
Algorithm \ref{alg:LSQ} requires $O(nd_{\max}+n\log(n))$ operations.
\end{theorem}
\begin{proof}
For Algorithm \ref{alg:LSQ}, its first step requires $O(nd_{\max})$, second step requires 
$O(nd_{\max}+n\log(n))$, where the $O(nd_{\max})$ part comes from matrix vector multiplication, and $O(n\log(n))$ part comes 
from sorting. 
For its third step, to avoid solving the normal equation exactly for large scale problems, we recommend using an iterative mehod, for example conjugate gradient descent (we use MATLAB's \emph{lsqr} operation in our implementation). As we have shown the matrices are associated with well behaved condition numbers, it requires only a constant number of iterations to get a well approximated least 
squares solution to problem (\ref{inalglsqremoveT}). Since the cost for each iteration in conjugate 
gradient descent equals to a few operations of matrix vector multiplication, which is $O(nd_{\max})$, the total cost for 
Algorithm \ref{alg:LSQ} is $O(nd_{\max}+n\log(n))$. 
\end{proof}

As a consequence, the total run time for Algorithm \ref{alg3} is $O(nd_{\max}+n\log(n))$. if the number of clusters $k=O(1)$, then Algorithm \ref{alg4} also runs in $O(nd_{\max}+n\log(n))$.

\begin{remark}
The computational scheme of our methods follow the similar framework as CP+RWT in \cite{LaiMckenzie2020}. However, one of the differences between these two approaches is that we apply \emph{lsqr} to solve the least squares problem (\ref{inalglsqremoveT}), but CP+RWT applies $O(\log n)$ iterations of subspace pursuit algorithm to solve (\ref{inalglsqremoveT}), and each its subspace pursuit is implemented with \emph{lsqr} as a subroutine. So essentially, our proposed method is $O(\log n)$ times cheaper than CP+RWT. We can also see this difference by comparing the run times for our numerical experiments in the next section.
\end{remark}

\section{Numerical Experiments}
\label{sectionnumerical}
In this section, we evaluate the performance of our algorithms on synthetic symmetric stochastic block model (SSBM), network data on political 
blogs \cite{AG05}, AT\&T Database of Faces \footnote{\url{https://git-disl.github.io/GTDLBench/datasets/att_face_dataset/}}, Extended Yale Face Database B (YaleB) \footnote{\url{http://vision.ucsd.edu/~leekc/ExtYaleDatabase/ExtYaleB.html}}, and MNIST data \footnote{\url{http://yann.lecun.com/exdb/mnist/}}.

For single cluster extraction tasks, we will consider the diffusion based methods plus a possible refinement procedure CP+RWT~\cite{LaiMckenzie2020}, HK-Grow \cite{Kloster2014}, PPR \cite{Andersen2007}, and LBSA \cite{Shi2019} as our baseline methods. For multi-cluster extraction tasks, we will consider ICP+RWT~\cite{LaiMckenzie2020}, LapRF and TVRF \cite{Yin2018}, Multi-class MBO Auction Dynamics \cite{Jacobs2018}, AtlasRBF \cite{Pitelis2014}, Pseudo-label \cite{Lee2013}, DGN \cite{Kingma2014} and Ladder Networks \cite{Rasmus2015} as the baseline methods. The standard spectral clustering algorithm \cite{Ng2002} is also being applied in some of the experiments.
For the implementation of Algorithms \ref{alg3} and \ref{alg4}, we use MATLAB's $lsqr$ function as our iterative least squares solver to solve equation (\ref{inalglsqremoveT}). We tune the rejection parameters $R$ for all algorithms appropriately to make the output $C^{\#}_i$ of each algorithm approximately the same size for comparison purpose. 
For the evaluation metrics of our experiments, we will consider Jaccard index for symmetric stochastic block model, F1 score for human faces data, and classification accuracy for political blogs data and MNIST data.
More implementation details are summarized as a supplementary material.

\subsection{Stochastic Block Model Data}
We first test Algorithm~\ref{alg3} on $SSBM(n,k,p,q)$ with different choices of parameters. The paramenter $n$ indicates the total number of vertices, $k$ indicates the number of clusters, $p$ is the probability of having an edge between any two  vertices within each cluster, and $q$ is the probability of having an edge between any two vertices from different clusters.  Figure~\ref{SSBMfigure} left panel demonstrates such a synthetic random graph model with three underlying clusters.  Figure~\ref{SSBMfigure} right panel illustrates an adjacency matrix of a random graph generated from symmetric stochastic block model with three underlying clusters. In our experiments, we fix $k=3$ and choose $n=600, 1200, 1800, 2400, 3000$ respectively. The connection probability between edges are chosen as $p=\frac{8\log n}{n}$ and $q=\frac{\log n}{n}$. By choosing the parameters carefully, we obtain the Jaccard index and logarithm of running time of each method shown in Figure~\ref{SSBMresults}. We also run the experiments on stochastic block model for non-symmetric case and obtained similar gaps in accuracy and run time. For the implementation of symmetric stochastic block model, we use three vertices with given label as our seeds, and we focus on only recovering the target cluster, say $C_1$. The experiments are performed with 500 repetitions.
\begin{figure}[htpb]
    \centering
    \begin{tabular}{cc}
    \includegraphics[width=0.45\textwidth]{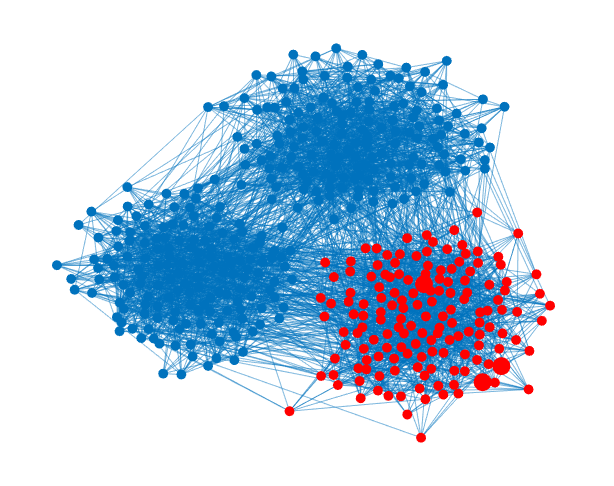} & 
	\includegraphics[width=0.45\textwidth]{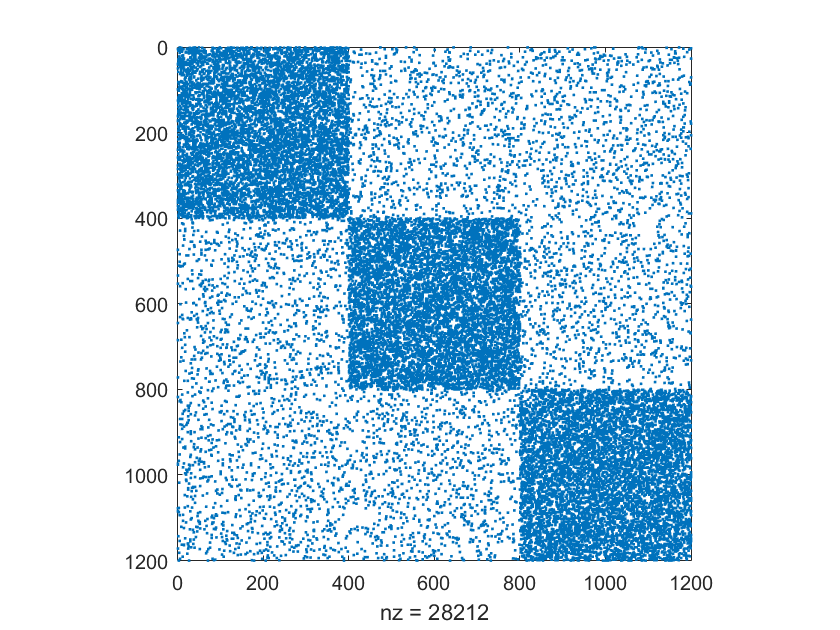}
    \end{tabular}
    \caption{\emph{Left}: A Random SSBM Graph with Three Underlying Clusters. \emph{Right}: Adjacency Matrix of A Random Graph Generated From SSBM.}
	\label{SSBMfigure}
\end{figure}  

\begin{figure}[htpb]
    \centering
    \begin{tabular}{cc}
    \includegraphics[width=0.45\textwidth]{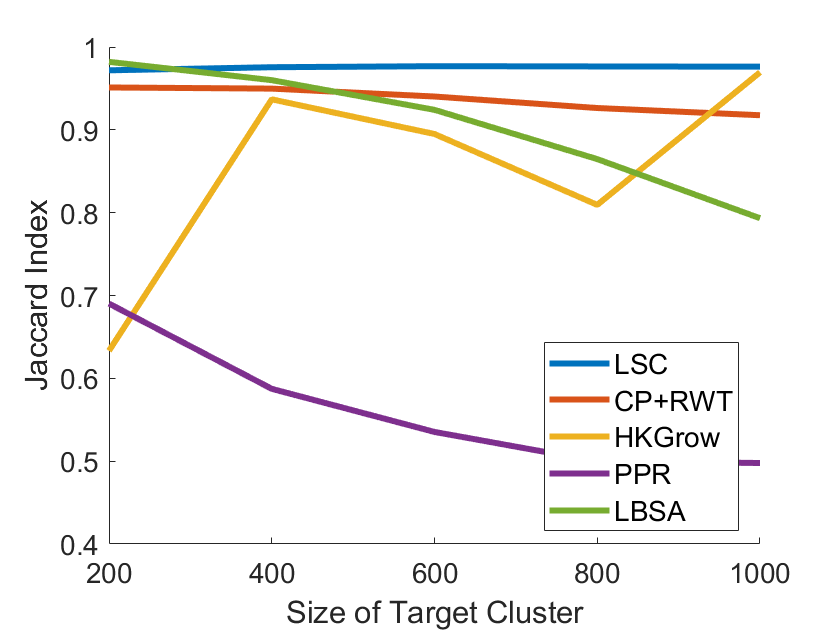} & 
	\includegraphics[width=0.45\textwidth]{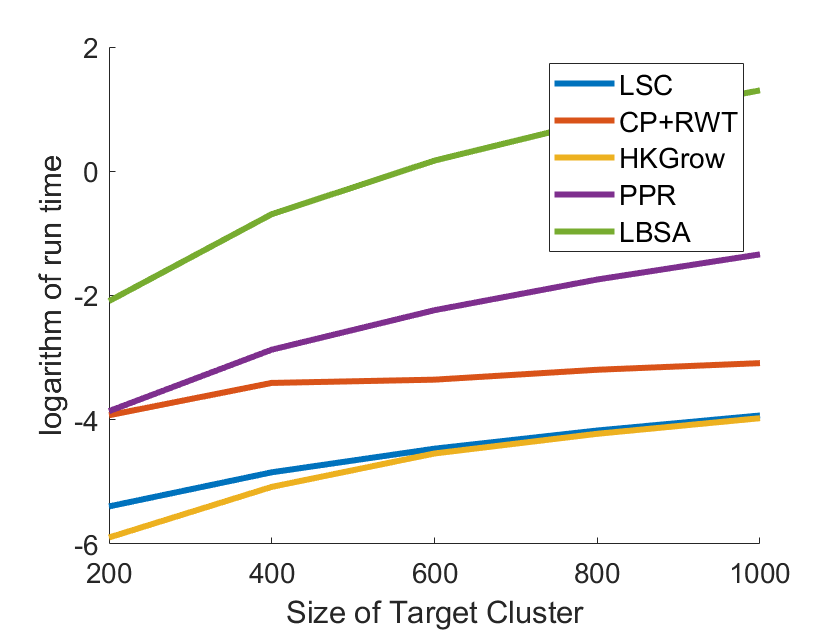}
    \end{tabular}
    \caption{\emph{Left}: Average Jaccard Index. \emph{Right}: Logarithm of the Average Run Time.}
	\label{SSBMresults}
\end{figure}  

% \begin{table}[ht] 
%     \small
% 	\caption{Performance of Finding $C_1$ Using LSC and CP+RWT with $k=3$ (Time is measured in seconds).}
% 	\label{SSBMtable}
% 	\centering
% 	\begin{tabular}{ccc|cc|cc}
%     	\hline
% 		$n$ & $p$ & $q$ & LSC & Time & CP+RWT~\cite{LaiMckenzie2020} & Time  \\
% 		\hline
% 		$300$ & $0.1$ & $0.005$ & $99.8\%$ & 0.005 & $99.6\%$ & 0.018\\
% 		$300$ & $0.1$ & $0.01$ & $97.7\%$ & 0.005 & $95.1\%$ & 0.014 \\
% 		\hline
% 		$1200$ & $0.1$ & $0.01$ & $100\%$ & 0.036 & $99.9\%$ & 0.052 \\
% 		$1200$ & $0.1$ & $0.02$ & $99.6\%$ & 0.038 & $97.4\%$ & 0.043 \\
% 		\hline
% 		$4800$ & $0.1$ & $0.01$ & $100\%$ & 0.371 & $100\%$ & 0.437 \\
% 		$4800$ & $0.1$ & $0.03$ & $99.9\%$ & 0.392 & $99.4\%$ & 0.502 \\
% 		$4800$ & $0.1$ & $0.035$ & $98.6\%$ & 0.394 & $92.7\%$  & 0.445  \\
% 		\hline
% 		$19200$ & $0.1$ & $0.01$ & $100\%$ & 4.30 & $100\%$ & 9.12  \\
%     	$19200$ & $0.1$ & $0.04$ & $99.9\%$ & 6.70 & $99.3\%$ & 11.35 \\
%     	$19200$ & $0.1$ & $0.045$ & $99.0\%$ & 7.72 & $93.8\%$ & 12.40  \\
% 		\hline
% 	\end{tabular}
% \end{table}    

\subsection{Political Blogs Network Data}
Next we test Algorithm \ref{alg3} on the data from "The political blogosphere and the 2004 US Election"\cite{AG05}, which contains a list of political blogs that were classified as liberal or conservative, and links between the 
blogs. See Figure~\ref{blog} as an illustration for the community structure (Figure source \cite{AG05}).
\begin{figure}[h]
	\centering
	\includegraphics[width=0.5\textwidth]{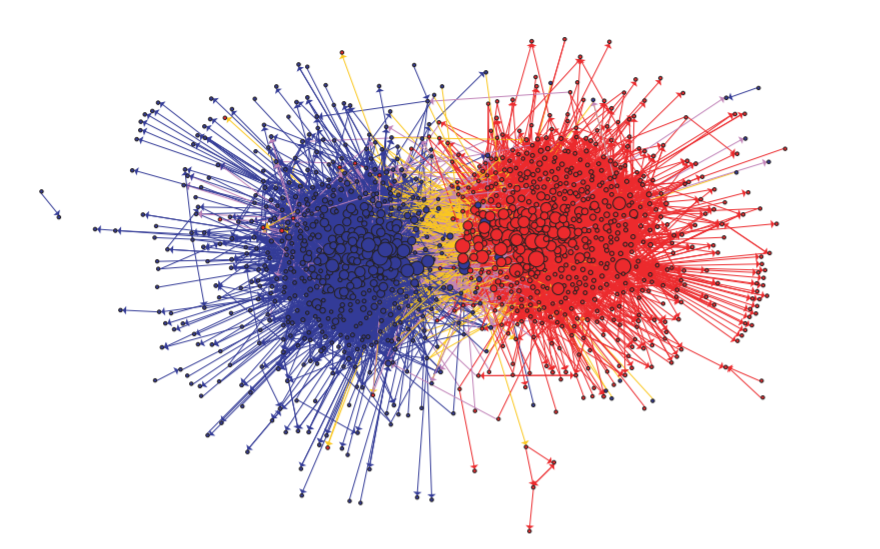}
    \caption{\small{Community structure of political blogs. The colors reflect political orientation, red for conservative, and blue for liberal. Orange links go from liberal to conservative, and purple ones from conservative to liberal. The size of each blog reflects the number of other blogs that link to it} \cite{AG05}. }
    \label{blog}
\end{figure}

As explained by Abbe and Sandon in \cite{Abbe2015}, their simplified algorithm gave a reasonably good classification 37 times out of 40 trials. Each of these trials classified all but 56 to 67 of the 1222 vertices in the graph  
main component correctly. According to \cite{Abbe2015}, the state-of-the-art described in \cite{CG15} before the work in \cite{Abbe2015} gives a lowest value around 60, while using regularized spectral algorithms such as the one in \cite{QR13} obtain about 80 errors. 
In our experiments, given three labeled seeds, the Algorithm \ref{alg3} succeeds 35 trials out of a total of 40 trials. Among these 35 trials, the average number of misclassified node in the graph main component is 55, which is slightly favorable than the state-of-the-art result. We also tested CP+RWT on this dataset and found the results were not very satisfactory.

\subsection{AT\&T Database of Faces}
The AT\&T Database of Faces 
contains gray scale images for $40$ different people of pixel size $92\times 112$. Images of each 
person are taken under $10$ different conditions, by varying the three perspectives of faces, lighting 
conditions, and facial expressions. 

\begin{figure}[htpb]	
	\centering
	\begin{tabular}{cc}
		\includegraphics[width=0.40\textwidth]{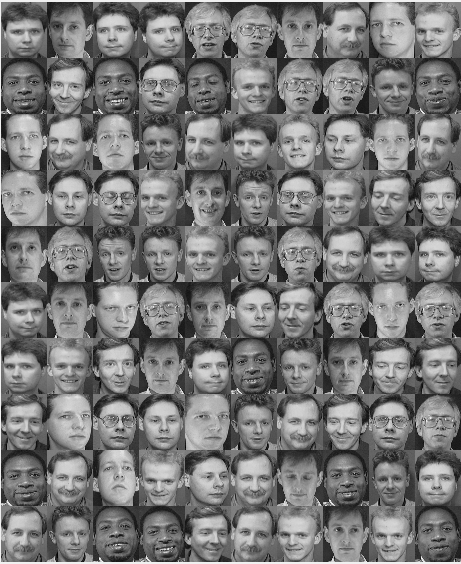} &
		\includegraphics[width=0.40\textwidth]{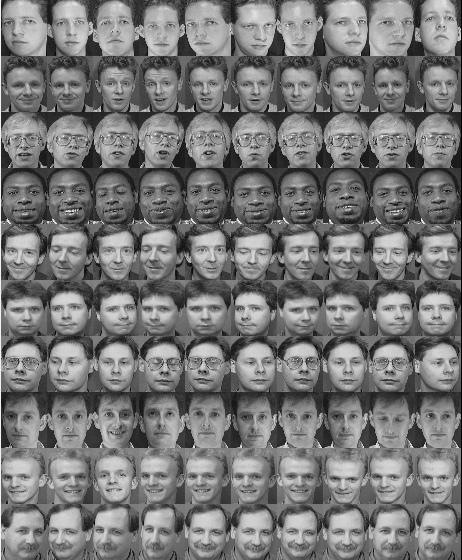}
	\end{tabular}
	\caption{\emph{Left}: Randomly Permuted AT\&T Faces. \emph{Right}: Desired Recovery of all Faces into Correct Clusters.}
	\label{FigureATT}
\end{figure}
We use part of this data set by randomly selecting 10 people such that each individual has 10 images. We randomly permute the 
images as shown in the left side of Figure~\ref{FigureATT}, and compute its adjacency matrix $A$ based on the preprocessing method discussed in the appendix \ref{preprocessing}.  Then we iteratively apply Algorithm~\ref{alg3} and try to recover all the 10 images 
belong to the same individual. The desired permutation of these individual images after iteratively performing Algorithm~\ref{alg3} are shown in the right side of Figure~\ref{FigureATT}. Some more details of the implementation regarding to the hyperparameters tuning are summarized in the appendxi \ref{parametertuning}.
The performance of our 
algorithm compared with CP+RWT and Spectral Clustering (SC) are summarized in Table~\ref{TableATT} under 500 repetitions. Note that spectral 
clustering method is unsupervised, hence its accuracy does not affected by the percentage of labeled data.
\begin{table}[htpb]
	\caption{Average F1 Scores of Recovering All Clusters for AT\&T Data.}
	\centering
	\label{TableATT}
	\begin{tabular}{c|ccc}
		\hline
		Labeled Data \%  & $10$ & $20$ & $30$  \\
		\hline
		LSC & $96.5\%$ & $97.5\%$ & $98.2\%$
		\\
		CP+RWT~\cite{LaiMckenzie2020} & $92.2\%$ & $95.7\%$ & $97.1\%$
		\\
		SC~\cite{Ng2002} & 95.8\% & 95.8\% & 95.8\% \\
		\hline
	\end{tabular}
\end{table}   
 
\subsection{Extended Yale Face Database B (YaleB)}
The YaleB dataset contains 16128 gray scale images of 28 human subjects under 9 poses and 64 illumination conditions. We use part of this data set by randomly selecting 20 images from each person after the preprocessing in appendix \ref{preprocessing}. 
The images are randomly permuted and we aim to recover all the clusters by iteratively performing Algorithm~\ref{alg3}. Figure~\ref{FigureYaleB}
shows this dataset with randomly permuted images on the left side and the desired clustering results on the right side. Figure~\ref{FigureYaleB_small} enlarges a small part inside the pictures from Figure~\ref{FigureYaleB} with the red boxes. The performance of our algorithm compared with CP+RWT and Spectral Clustering (SC) are summarized in Table~\ref{TableYaleB} under 500 repetitions.

\begin{table}[htpb] 
	\caption{Average F1 Scores of Recovering All Clusters for YaleB Data.}
	\label{TableYaleB}
	\centering
	\begin{tabular}{c|ccc}
		\hline
		Labeled Data \%  & $5$ & $10$ & $20$   \\
		\hline
		LSC & 92.1\% & $96.0\%$ & $96.1\%$  	\\
		CP+RWT~\cite{LaiMckenzie2020} & 89.2\% & $93.7\%$ & $93.9\%$ 	\\
		SC~\cite{Ng2002} & 93.8\% & 93.8\% & 93.8\% \\
		\hline
	\end{tabular}
\end{table}    

\begin{figure}[htpb]	
	\centering
	\begin{tabular}{cc}
	    \includegraphics[width=0.45\textwidth]{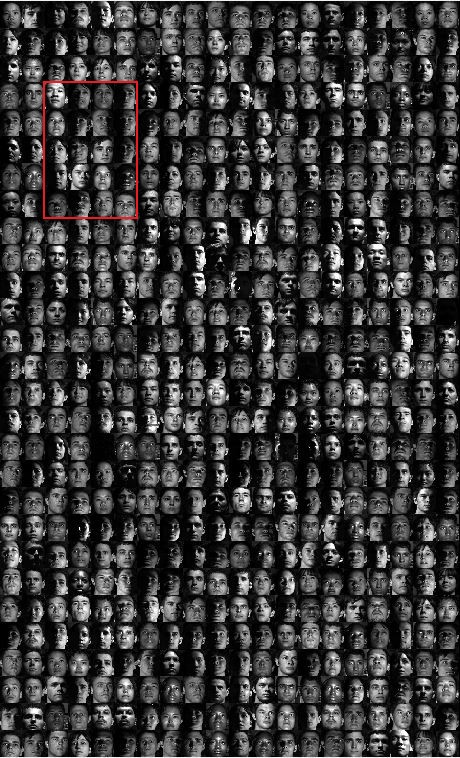} &
	    \includegraphics[width=0.45\textwidth]{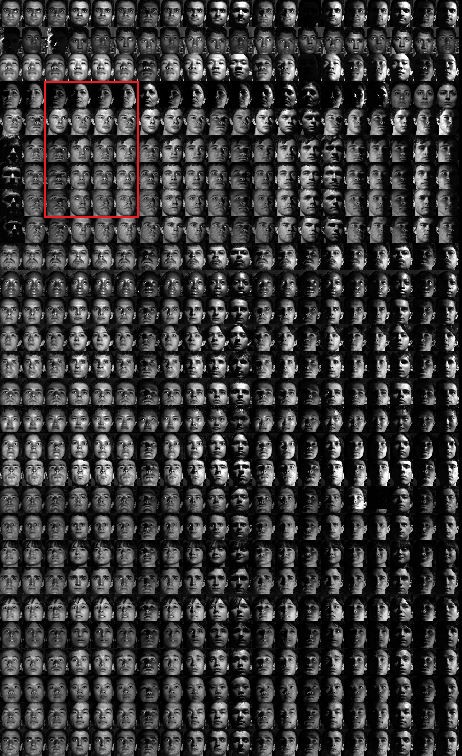}
	\end{tabular}
	\caption{\emph{Left}: Randomly Permuted YaleB Faces. \emph{Right}: Desired Recovery of all Faces into Correct Clusters.}
	\label{FigureYaleB}
\end{figure}

\begin{figure}[htpb]	
	\centering
	\begin{tabular}{cc}
	    \includegraphics[width=0.255\textwidth]{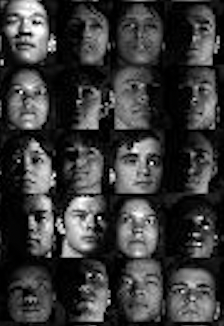} &
	    \includegraphics[width=0.26\textwidth]{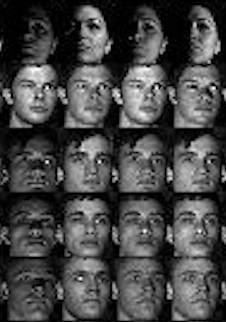}
	\end{tabular}
	\caption{Enlarged YaleB Human Faces.}
	\label{FigureYaleB_small}
\end{figure}

\subsection{MNIST Data}
We also test Algorithm~\ref{alg4} on the MNIST data, which is a famous machine learning benchmark dataset in classification that consists of $70000$ grayscale images of the handwritten digits $0$-$9$ of size $28\times 28$ with approximately $7000$ images of each digit. 
%\begin{figure}[htpb]	
%	\centering
%	\begin{tabular}{cc}
%		\includegraphics[width=0.32\textwidth]{mnistrandperm.png} &
%		\includegraphics[width=0.315\textwidth]{mnistsorted.png}
%	\end{tabular}
%	\caption{MNIST Data (random permutation on the left and perfect recovering on the right) }
%	\label{FigureATT}
%\end{figure}
We used a certain percentage of vertices with given labels within each of the ten clusters as our seed vertices, the performance ILSC and ICP+RWT are summarized in Table \ref{oneperent} under 100 repetitions.
%\begin{table}[htpb] 
%	\caption{Performance of ILSC and ICP+RWT for Finding All Clusters in MNIST with Labeled Data.}
%	\label{oneperent}
%	\centering
%	\begin{tabular}{cccccc}
%		\hline
%		Labeled Data \% & $0.5$ & $1$ & $1.5$ & $2$ & $2.5$  \\
%		\hline
%		ILSC & $97.30\%$ & $97.73\%$ & $98.03\%$ & $98.17\%$ & $98.27\%$ \\
%		Time & 15.5 s & 15.3 s & 15.4 s & 15.5 s & 15.4 s \\
%		\hline
%		ICP+RWT & $96.41\%$ & $97.32\%$ & $97.44\%$ & $97.52\%$ & $97.50\%$ \\
%		Time & 18.1 s & 19.1 s & 19.8 s & 21.4 s & 22.1 s \\
%		\hline
%	\end{tabular}
%\end{table}    

\begin{table}[htpb] 
	\caption{Average Accuracy of Recovering All Clusters for MNIST Data (Time is measured in seconds).}
	\label{oneperent}
	\centering
	\begin{tabular}{c|cc|cc}
		\hline
		Labeled Data \% & ILSC & Run Time & ICP+RWT~\cite{LaiMckenzie2020} & Run Time \\
		\hline
		0.5 & $97.30\%$ & 15.5 & $96.41\%$ & 18.1 \\
		1 & $97.73\%$ & 15.3 & $97.32\%$ & 19.1 \\
		1.5 & $98.03\%$  & 15.4 & $97.44\%$ & 19.8 \\
		2 & $98.17\%$ & 15.5 & $97.52\%$ & 21.4 \\
		2.5 & $98.27\%$ & 15.4 & $97.50\%$ & 22.1 \\
		\hline
	\end{tabular}
\end{table}   

%\begin{table}[htpb] 
%	\caption{Running Time For Extracting all Clusters Using ILSC and ICP+RWT with Labeled Data.}
%	\label{MNISTtime}
%	\centering
%	\begin{tabular}{cccccc}
%		\hline
%		Labeled Data \% & $0.5$ & $1$ & $1.5$ & $2$ & $2.5$  \\
%		\hline
%		ILSC & 15.5 s & 15.3 s & 15.4 s & 15.5 s & 15.4 s
%		\\
%		ICP+RWT & 18.1 s & 19.1 s & 19.8 s & 21.4 s & 22.1 s \\
%		\hline
%	\end{tabular}
%\end{table}    

%\begin{table}[ht] 
%	\caption{Accuracy of ILSC and ICP+RWT for Finding All Clusters in MNIST with Medium Percentage Seed Vertices.}
%	\label{fiftypercent}
%	\centering
%	\begin{tabular}{cccccc}
%		\hline
%		Labeled Data \% & $10$ & $20$ & $30$ & $40$ & $50$  \\
%		\hline
%		ILSC  & $99.0\%$ & $99.1\%$ & $99.2\%$ & $99.3\%$ & %99.3\%$
%		\\
%		ICP+RWT & $98.7\%$ & $98.9\%$ & $99.0\%$ & $99.1\%$ & $99.2\%$
%		\\
%		LASSO & $0\%$ & $0\%$ & $0\%$ & $0\%$ & $0\%$
%		\\
%		\hline
%	\end{tabular}
%end{table}    
We also compare ILSC with several other state-of-the-art semi-supervised methods on MNIST. As we can see in Table~\ref{Ladder}, ILSC outperforms the other algorithms except for the Ladder Networks which uses more information of labels and involved in a deep neural network architecture that requires training on GPUs for several hours.
\begin{table}[htpb] 
	\caption{Accuracy of ILSC and other Semi-supervised Algorithms on MNIST.}
	\label{Ladder}
	\centering
	\begin{tabular}{ccc}
		\hline
		Methods & \# Labeled Data & Accuracy  \\
		\hline
		LapRF \cite{Yin2018} & $600$ & $95.6\%$ 
		\\
		TVRF \cite{Yin2018} & $600$ & $96.8\%$ 
		\\
		ICP+RWT \cite{LaiMckenzie2020} & $700$ & $97.3\%$ 
		\\
		Multi-class MBO with Auction Dynamics \cite{Jacobs2018}  & $700$ & $97.4\%$ 
		\\
		ILSC (this paper) & $700$ & $97.7\%$ 
		\\
		AtlasRBF \cite{Pitelis2014} & $1000$ & $96.4\%$ 
		\\
		Pseudo-label \cite{Lee2013} & $1000$ & $96.6\%$ 
		\\
		DGN \cite{Kingma2014} & $1000$ & $97.6\%$ 
		\\
		ILSC (this paper) & $1000$ & $98.0\%$ 
		\\
		Ladder Networks \cite{Rasmus2015} & $1000$ & $99.2\%$
		\\
		\hline
	\end{tabular}
\end{table}

\vspace{0.5cm}
\noindent
\textbf{Declarations} 

\paragraph{\textbf{\emph{\small Funding}}} \small{The first author is supported by the Simons Foundation collaboration grant \#864439.}

\paragraph{\textbf{\emph{\small Competing Interests}}} \small{The authors have disclosed any competing interests.}

\paragraph{\textbf{\emph{\small Data Availability}}} \small{A sample demo program for reproducing Fig. \ref{FigureATT} in this paper can be found at \url{ https://github.com/zzzzms/LeastSquareClustering}. All other demonstration codes or data are available upon request.}

\appendix
\section{Hyperparameters for Numerical Experiments}
\label{parametertuning}
For each cluster to be recovered, we sampled the seed vertices $\Gamma_i$ uniformly from $C_i$ during all implementations. We fix the random walk depth with $t=3$, use random walk threshold parameter $\delta=0.8$ for political blogs network and $\delta=0.6$ for all the other experiments. We vary the rejection parameters $R\in (0,1)$ for each specific experiments based on the estimated sizes of clusters. In the case of no knowledge of estimated sizes of clusters nor the number of clusters are given, we may have to refer to the spectra of graph Laplacian and use the large gap between two consecutive spectrum to estimate the number of clusters, and then use the average size to estimate the size of each cluster.  

We fix the least squares threshold parameter with $\gamma=0.2$ for all the experiments, which is totally heuristic. However, we have experimented that the performance of algorithms will not be affected too much by varying $\gamma\in [0.15,0.4]$. The hyperparameter $MaxIter$ is chosen according to the size of initial seed vertices relative to the total number of vertices in the cluster. For purely comparison purpose, we keep $MaxIter=1$ for MNIST data. By experimenting on different choices of $MaxIter$, we find that the best performance for AT\&T data occurs at $MaxIter=2$ for $10\%$ seeds and $MaxIter=1$ for $20\%$ and $30\%$ seeds. For YaleB data, the best performance occurs at $MaxIter=2$ for $5\%$, $10\%$, and $20\%$ seeds. All the numerical experiments are implemented in MATLAB and can be run on a local personal machine, for the authenticity of our results, we put a sample demo code at \emph{https://github.com/zzzzms/LeastSquareClustering} for verification purpose.

\section{Image Data Preprocessing} \label{preprocessing}
For YaleB human faces data, we have performed some data preprocessing techinuqes to avoid the poor quality images. Specifically, we abandoned the pictures which are too dark, and we cropped each image into size of $54\times 46$ to reduce the effect of background noise. For the remaining qualified pictures, we randomly selected 20 images for each person. 

All the image data in MNIST, AT\&T, YaleB needs to be firstly constructed into an auxiliary graph before the implementation. Let $\mathbf{x}_i\in\mathbb{R}^n$ be the vectorization of an image from the original data set, we define the following affinity matrix of the $K$-NN auxiliary graph based on Gaussian kernel according to \cite{Jacobs2018} and \cite{Zelnik2004},. 
\[ A_{ij} = \begin{cases} 
e^{-\|\mathbf{x}_i-\mathbf{x}_j\|^2/\sigma_i\sigma_j} & \text{if} \quad \mathbf{x}_j\in NN(\mathbf{x}_i,K) \\
0 & \text{otherwise} 
\end{cases}
\]
The notation $NN(\mathbf{x}_i,K)$ indicates the set of $K$-nearest neighbours of $\mathbf{x}_i$, and $\sigma_i:=\|\mathbf{x}_i-\mathbf{x}^{(r)}_i\|$ where $\mathbf{x}^{(r)}_i$ is the $r$-th closest point of $\mathbf{x}_i$. Note that the above $A_{ij}$ is not necessary symmetric, so we consider $\tilde{A}_{ij}=A^T A$ for symmetrization. Alternatively, one may also want to consider $\tilde{A}=\max\{A_{ij}, A_{ji}\}$ or $\tilde{A}=(A_{ij}+A_{ji})/2$. We use $\tilde{A}$ as the input adjacency matrix for our algorithms. 

We fix the local scaling parameters $K=15$, $r=10$ for the MNIST data, $K=8$, $r=5$ for the YaleB data, and $K=5$, $r=3$ for the AT\&T data during implementation.

\end{document}